\documentclass{amsart}

\usepackage[utf8]{inputenc}   
\usepackage[margin=1in]{geometry} 
\usepackage{amsmath, amssymb} 
\usepackage{amsthm}           
\usepackage{graphicx}         
\graphicspath{{../}}          
\usepackage{subcaption}       
\usepackage{booktabs}         
\usepackage{multirow}         
\usepackage{array}            
\usepackage{longtable}        
\usepackage{cite}             
\usepackage{color}            

\ifdefined\REDLINE
  \newcommand{\rev}[1]{{\color{red}#1}}
\else
  \newcommand{\rev}[1]{#1}
\fi

\newtheorem{theorem}{Theorem}
\newtheorem{lemma}[theorem]{Lemma}
\newtheorem{proposition}[theorem]{Proposition}
\newtheorem{corollary}[theorem]{Corollary}
\newtheorem{definition}[theorem]{Definition}
\newtheorem{Remark}{Remark}

\numberwithin{equation}{section}
\allowdisplaybreaks
\begin{document}

\title[Weighted $L^2$ PINN Loss for solving the BGK model]{A Theory-guided Weighted $L^2$ Loss for solving the BGK model via Physics-informed neural networks}

\author{Gyounghun Ko}
\address{Academy of Mathematics and Systems Sciences, Chinese Academy of Sciences, Beijing, 100190, China}
\thanks{G. Ko and S.-J. Son contributed equally to this work.}

\author{Sung-Jun Son}
\address{Center for Mathematical Machine Learning and its Applications (CM2LA), Department of Mathematics, Pohang University of Science and Technology}

\author{Seung Yeon Cho}
\address{Department of Mathematics, Gyeongsang National University}

\author{Myeong-Su Lee}
\address{Research Institute of Mathematics, Seoul National University}
\thanks{* Corresponding Author: Myeong-Su Lee, E-mail: msl3573@snu.ac.kr}

\begin{abstract}
While Physics-Informed Neural Networks offer a promising framework for solving partial differential equations, the standard $L^2$ loss formulation is fundamentally insufficient when applied to the Bhatnagar--Gross--Krook (BGK) model. Specifically, simply minimizing the standard loss does not guarantee accurate predictions of the macroscopic moments, causing the approximate solutions to fail in capturing the true physical solution. To overcome this limitation, we introduce a velocity-weighted $L^2$ loss function designed to effectively penalize errors in the high-velocity regions. By establishing a stability estimate for the proposed approach, we show that minimizing the weighted loss guarantees the convergence of the approximate solution. Moreover, numerical experiments demonstrate that employing this weighted PINN loss leads to superior accuracy and robustness across various benchmarks compared to the standard approach.
\end{abstract}
\maketitle

\section{Introduction}\label{sec:intro}
Kinetic equations describe non-equilibrium transport phenomena through the time evolution of a velocity distribution function in a phase space involving time, physical space, and velocity space. Such equations are essential in regimes where continuum fluid models become inaccurate, and they play a key role in numerous engineering applications, including hypersonic and high-altitude aerodynamics, vacuum technology, micro/nano-scale gas flows, and related multiscale transport problems~\cite{cercignani1988boltzmann,sone2007molecular,bird1994molecular,dimarco2014numerical}. However, traditional grid-based solvers often incur prohibitive computational costs due to the high dimensionality of the phase space and the need to evaluate velocity integrals for macroscopic moments.

Recently, Physics-Informed Neural Networks (PINNs) have emerged as a promising deep learning framework for solving diverse partial differential equations~\cite{raissi2019physics,karniadakis2021physics,pang2019fpinns}. A key idea of PINN is to reformulate the PDE problem, together with initial and boundary conditions, into an optimization problem over neural network parameters. In the standard PINN setup, the optimization objective---commonly referred to as the PINN loss $\mathcal{L}_{\mathrm{PINN}}$---is typically constructed as a linear combination of squared $L^2$ norms of the residuals. By evaluating these residual terms at randomly sampled points and optimizing the network via stochastic gradient methods, PINNs offer a practical pathway to bypass the curse of dimensionality~\cite{han2018solving,e2018deep,e2022algorithms}. Motivated by this advantage, PINNs have been actively explored for high-dimensional kinetic equations~\cite{hwang2020trend, lee2023oppinn, li2021physics, kim2026physics, zhang2023physics,jin2024asymptotic,jin2024asymptotic2,lou2021physics}. Furthermore, recent methodological advancements, such as the application of separable architectures~\cite{cho2023separable,oh2025separable, chen2026micro} and neural sparse representations~\cite{li2024solving}, have efficiently addressed the  difficulty of evaluating macroscopic moments via numerical integration over the velocity space. However, despite these empirical successes across various benchmark problems in kinetic theory, the theoretical understanding regarding the reliability of the resulting approximate solutions remains highly limited.

This issue of reliability becomes particularly critical when considering the practical limitations of the optimization process. Ideally, if the optimization objective $\mathcal{L}_{\mathrm{PINN}}$ could be driven exactly to zero, the approximate solution would perfectly coincide with the true solution. In practice, however, it is generally difficult to drive $\mathcal{L}_{\mathrm{PINN}}$ exactly to zero due to several fundamental limitations: the finite expressive capacity of neural networks, numerical errors originating from quadrature or sampling, and the inherent difficulties of non-convex optimization~\cite{wang2021understanding,krishnapriyan2021characterizing}. Consequently, the quality of the learned solution is typically understood by whether the training loss is sufficiently small. As discussed in~\cite{wang20222,de2024numerical}, this naturally raises a fundamental question:
\begin{center}
\textbf{``Does a small standard $L^2$ PINN loss $\mathcal{L}_{\mathrm{PINN}}$ ensure an accurate solution?''}
\end{center}
The answer to this question is intimately connected to the mathematical stability of the target PDE, meaning it is highly problem-specific. For example, in linear elliptic or parabolic PDEs, classical stability theory ensures that a small residual (in an appropriate norm) directly translates to a small error in the solution~\cite{de2024numerical,evans2022partial,gilbarg1998elliptic}. In contrast, the authors in~\cite{wang20222} have revealed that for high-dimensional Hamilton-Jacobi-Bellman (HJB) equations, the standard $L^2$ loss fails to guarantee accuracy, suggesting an alternative training method. These contrasting examples imply that the adequacy of a given PINN loss formulation is not universally guaranteed. Therefore, its validity must be carefully examined for specific classes of equations. Related theoretical results for PINN applied to the Boltzmann equation have also been established under near-equilibrium regimes~\cite{abdo2024error}. For the general non-equilibrium setting, however, the mathematical understanding of the standard PINN loss remains largely unexplored.

In this paper, we study the above question for the Bhatnagar--Gross--Krook (BGK) model~\cite{bhatnagar1954model,yun2010cauchy}, a relaxation model of the Boltzmann equation. The BGK model replaces the Boltzmann collision operator by a relaxation toward a local Maxwellian, which is determined by macroscopic moments up to second order (mass, momentum, and energy) of the distribution function. These moments are velocity integrals weighted by $1$, $v$, and $|v|^2$. Due to this nonlocal moment coupling, even a small error in the high-velocity tail can significantly bias the macroscopic moments, and consequently the approximation may relax toward an incorrect local equilibrium state. This indicates that, for the BGK model, a small standard $L^2$-based loss may not be sufficient to ensure the accuracy of physically relevant quantities. Indeed, in Section~\ref{sec:counterexamples}, we make this issue concrete by constructing explicit \rev{counterexamples---families} of approximate solutions $\{\tilde f_\varepsilon\}_{\varepsilon>0}$ such that $\mathcal{L}_{\mathrm{PINN}}(\tilde f_\varepsilon)=\mathcal{O}(\varepsilon^2)$, while $\tilde f_\varepsilon$ does not converge to the true solution $f$ as $\varepsilon\to 0$. This shows that the answer to the above question can be negative for the BGK model, and that the standard $L^2$ PINN loss may not be a reliable proxy for solution accuracy.

To resolve this issue, in Section~\ref{sec:stability}, we introduce a weighted PINN loss $\mathcal{L}_{w\text{-PINN}}$ motivated by the vulnerabilities identified in the counterexamples. By weighting the standard $L^2$ residual terms with a velocity-dependent function $w(v)$, we can effectively penalize errors in the high-velocity region and sufficiently control the macroscopic moments. While this naturally rules out the specific spurious solutions presented in Section~\ref{sec:counterexamples}, we go a step further to ensure its overall reliability. To this end, we rigorously establish a weighted stability estimate for the BGK model within a suitable ansatz space. Under some integrability conditions on $w$, our stability analysis yields a rigorous convergence guarantee: if $\mathcal{L}_{w\text{-PINN}}(\tilde f)\to 0$, then
\begin{align*}
\|(f-\tilde f)(t)\|_{L^2_{x,v}} \to 0 \qquad \text{for all } t\in[0,\mathcal{T}],
\end{align*}
where $\mathcal{T}$ is the terminal time. Moreover, these integrability conditions can be satisfied by simple and practical choices of weights. A representative example is the polynomial weight $w(v)=1+\alpha|v|^\beta$ with $\alpha>0$ and $\beta>7/2$. \rev{Finally, in Section~\ref{sec:numerical_experiments}, we validate the effectiveness of the proposed loss through extensive numerical experiments.}

The main contributions of this paper are summarized as follows:
\begin{itemize}
    \item We show that the standard $L^2$ PINN loss can be misleading for the BGK model. In particular, we construct explicit families of approximate solutions for which $\mathcal{L}_{\mathrm{PINN}}\to 0$ while the solution error does not vanish.

    \item We propose a weighted PINN loss $\mathcal{L}_{w\text{-PINN}}$ designed to penalize high-velocity errors, which successfully rules out the spurious families of solutions constructed in Section~\ref{sec:counterexamples}. To ensure its fundamental reliability, we then establish a weighted stability theory, proving that a vanishing weighted loss rigorously guarantees the $L^2$ convergence of the approximate solution and the $L^1$ convergence of its macroscopic moments.

    \item \rev{We demonstrate that the failure predicted by the counterexamples indeed occurs in actual training. That is, minimizing the standard $L^2$ loss down to a sufficiently small value still fails to yield accurate predictions. In contrast, under the same setting, sufficient minimization of the proposed weighted loss yields accurate predictions. Furthermore, through extensive benchmarks, we show that the proposed weighted loss consistently achieves the best accuracy among the compared losses at no additional computational cost.}
\end{itemize}

\section{Preliminaries}

\subsection{The BGK Model of the Boltzmann Equation}
In rarefied gas dynamics, the Boltzmann equation plays a pivotal role in describing the non-equilibrium behavior of gas flows. It is widely applied in various engineering and physical scenarios, such as the aerodynamics of space vehicles during atmospheric re-entry, vacuum technology, and micro-electro-mechanical systems (MEMS). Here, the state of gas particles is characterized by a velocity distribution $f(t, x, v)$ in the phase space $\Omega \times \mathbb{R}^3$ at time $t$, where the spatial domain is $\Omega \subset \mathbb{R}^3$ and the microscopic velocity is $v \in \mathbb{R}^3$. The time evolution of $f$ is governed by continuous transport and binary collisions among particles.

Due to the high-dimensional, five-fold integral nature of the Boltzmann collision operator, numerical simulations of the full Boltzmann equation are computationally prohibitive~\cite{dimarco2014numerical,filbet2006solving,gamba2017fast}. To alleviate this complexity, the Bhatnagar--Gross--Krook (BGK) model~\cite{bhatnagar1954model} was introduced, replacing the intricate collision integral with a simpler relaxation operator that drives the system toward a local thermodynamic equilibrium. The BGK model is given by:
\begin{equation*}
    \partial_t f + v \cdot \nabla_x f = \frac{1}{\mathrm{Kn}} \left( \mathcal{M}[f] - f\right).
\end{equation*}
The Knudsen number, denoted as $\mathrm{Kn}$, serves as a dimensionless parameter representing the ratio of the mean free path of a particle to the physical length scale of interest. In this context, we assume a fixed collision frequency. The local Maxwellian distribution, denoted as $\mathcal{M}[f]$, defines the particle density at local thermodynamic equilibrium and is characterized by:
\[
    \mathcal{M}[f](t,x,v) \equiv \mathcal{M}_{\left(\rho(t,x),u(t,x),T(t,x)\right)}(v) := \frac{\rho(t,x)}{\left(2\pi T(t,x)\right)^{3/2}}e^{-\frac{|v - u(t,x)|^2}{2T(t,x)}},
\]
where the macroscopic mass $\rho$, velocity $u$, and temperature $T$ are uniquely determined by the moments of the density function $f$:
\begin{align*}
    &\rho(t,x) = \int_{\mathbb{R}^3} f(t,x, v) dv, \\
    &u(t,x) = \frac{1}{\rho(t,x)}\int_{\mathbb{R}^3} v f(t,x, v) dv, \\
    &T(t,x) = \frac{1}{3\rho(t,x)}\int_{\mathbb{R}^3} |v-u(t,x)|^2 f(t,x, v) dv.
\end{align*}

In this paper, we restrict our attention to the spatial domain $\Omega = [0,1]^3$ subject to periodic boundary conditions. Specifically, we consider the following Cauchy problem for the BGK model:
\begin{align}
\begin{split}\label{eq:bgk}
    \partial_t f + v \cdot \nabla_x f &= \frac{1}{\mathrm{Kn}} \left( \mathcal{M}[f] - f\right), \quad \forall (t,x,v) \in (0,\mathcal{T})\times [0,1]^3 \times \mathbb{R}^3, \\
    f(0, x, v) &= f_0(x,v), \quad \forall (x, v) \in [0,1]^3 \times \mathbb{R}^3, 
\end{split}
\end{align}
with periodic boundary conditions.

Given its physical fidelity and relative simplicity compared to the full Boltzmann equation, the BGK model has been the subject of extensive mathematical and numerical studies. For mathematical analysis of the model, interested readers are referred to~\cite{yun2010cauchy, yun2015ellipsoidal,Yun2019,kim2021stationary, cho2025kinetic,saint2002modele,SY2022,SY2023,Son2025,SY2025}. Concurrently, numerous numerical methods have been developed and rigorously analyzed to efficiently simulate the BGK model, including various deterministic discrete-velocity methods, semi-Lagrangian schemes, and stochastic particle methods~\cite{boscarino2020high,cho2024conservative,dimarco2014numerical,filbet2010class,hu2018asymptotic,mieussens2000discrete,pieraccini2007implicit, sands2025adaptive}.

\subsection{Physics-Informed Neural Networks}

In recent years, PINNs have garnered significant attention for solving a wide spectrum of forward and inverse problems in physics and engineering~\cite{lee2025forward, gie2024semi,kim2024physics,kim2026physics,cai2021physics,jo2026task}. 

To formulate the PINN approach for the BGK model, let $f_\theta(t, x, v)$ be an approximate solution constructed using deep neural networks, parameterized by a set of weights and biases $\theta$. The primary goal of PINN is to find network parameters $\theta$ such that the approximate solution $f_\theta$ accurately satisfies the Cauchy problem~\eqref{eq:bgk}. \rev{For this, we define the following residuals:
\begin{equation}\label{eq:residuals}
\rev{\begin{aligned}
    \mathrm{Res}_{pde}(f_\theta) &:= \partial_t f_\theta + v \cdot \nabla_x f_\theta -\frac{1}{\mathrm{Kn}} \big( \mathcal{M}[f_\theta] - f_\theta \big), \\
    \mathrm{Res}_{bc,i}(f_\theta) &:= f_\theta(t, x|_{x_i=1}, v) - f_\theta(t, x|_{x_i=0}, v), \qquad i=1,2,3, \\
    \mathrm{Res}_{ini}(f_\theta) &:= f_\theta(0, x, v) - f_0(x,v),
\end{aligned}}
\end{equation}
where $x|_{x_i=0}$ and $x|_{x_i=1}$ denote points on the opposite boundary faces $x_i=0$ and $x_i=1$. We note that they measure how well $f_\theta$ satisfies the governing equation, the boundary conditions, and the initial condition, respectively. PINN then trains the approximate solution by minimizing the standard loss function $\mathcal{L}_{\mathrm{PINN}}(\theta)$:}
\begin{equation}\label{eq:pinn_loss}
    \mathcal{L}_{\mathrm{PINN}}(\theta) =\mathcal{L}_{pde}(\theta) + \lambda_{bc}\mathcal{L}_{bc}(\theta) + \lambda_{ini} \mathcal{L}_{ini}(\theta),
\end{equation}
\rev{where each loss component is the squared $L^2$ norm of the corresponding residual:}
\begin{equation}\label{eq:L2_components}
\rev{
    \mathcal{L}_{pde}(\theta) := \big\| \mathrm{Res}_{pde}(f_\theta) \big\|_{L^2_{t,x,v}}^2, \qquad
    \mathcal{L}_{bc}(\theta) := \sum_{i=1}^3 \big\| \mathrm{Res}_{bc,i}(f_\theta) \big\|_{L^2_{t,\hat{x}_i,v}}^2, \qquad
    \mathcal{L}_{ini}(\theta) := \big\| \mathrm{Res}_{ini}(f_\theta) \big\|_{L^2_{x,v}}^2,
}
\end{equation}
\rev{where a subscript on a norm indicates the variables over which the norm is taken, so that $\|\cdot\|_{L^2_{t,x,v}}$ is the $L^2$ norm over $(0,\mathcal{T}) \times [0,1]^3 \times \mathbb{R}^3_v$. Here $\hat{x}_i$ denotes the remaining spatial coordinates other than the $i$-th spatial coordinate $x_i$, and $\lambda_{bc}, \lambda_{ini} > 0$ are penalty weights balancing the contributions of the individual loss terms.} In practice, the \rev{integrals defining}~\eqref{eq:L2_components} are approximated using numerical integration or Monte Carlo sampling over a set of collocation points. One then seeks a set of optimized parameters $\theta^*$ by minimizing $\mathcal{L}_{\mathrm{PINN}}(\theta)$ using gradient-based optimization algorithms, such as Adam~\cite{kingma2014adam}.

Ideally, if $\mathcal{L}_{\mathrm{PINN}}(\theta)$ could be driven exactly to zero, the network would perfectly represent the true solution. However, achieving a strictly zero loss is practically impossible due to several fundamental limitations: the finite expressive capacity of the chosen neural network architecture, the inevitable discrepancy between the continuous exact integrals and their numerical approximations, and the inherent difficulties of non-convex optimization that often trap the solution in local minima. Consequently, the training process inevitably terminates with an approximate solution possessing a small, yet strictly positive, residual loss. This practical reality underscores the need to investigate whether a small standard $L^2$ loss inherently guarantees a high-fidelity approximation—a question we explore in detail in the following section.

\section{Limitations of the Standard $L^2$ PINN Loss}
\label{sec:counterexamples}

In this section, we address the following question posed in Section~\ref{sec:intro}:
whether a small standard $L^2$ PINN loss guarantees a small $L^2$ solution error.
We answer this question in the negative by constructing explicit families of approximate solutions $\{\tilde f_\varepsilon\}_{\varepsilon>0}$
for which the PINN loss converges to zero as $\varepsilon\rightarrow0$, while the corresponding $L^2$ error with respect to the exact solution remains bounded away from zero.

Although our main focus is the full inhomogeneous BGK model~\eqref{eq:bgk}, it suffices to construct a counterexample
in the spatially homogeneous setting:
\begin{align}\label{eq:homo_bgk}
\begin{split}
\partial_t f(t,v) &= \mathcal{M}[f(t,\cdot)](v) - f(t,v), \qquad (t,v)\in(0,\mathcal{T})\times\mathbb{R}^3,\\
f(0,v) &= f_0(v), \qquad v\in\mathbb{R}^3,
\end{split}
\end{align}
where $\mathcal{T}$ is a terminal time. \rev{For simplicity, we set $\mathrm{Kn}=1$ throughout this section, without loss of generality.} Indeed, suppose the initial data in the inhomogeneous BGK model is independent of $x$. Let $f(t,v)$ denote the corresponding solution of the spatially homogeneous BGK model with the same initial datum. If we view $f$ as a function on $(t,x,v)$ by setting $f(t,x,v):=f(t,v)$, then $v\cdot\nabla_x f \equiv 0$. Hence $f$ satisfies the inhomogeneous BGK model as well (with $x$-independent initial data). Therefore, any counterexample constructed in the homogeneous setting immediately yields a counterexample for the full model, which is sufficient to answer the fundamental question in the negative. \rev{We also note that in this homogeneous setting, the standard PINN loss reduces to}
\begin{align}\label{eq:pinn_loss_homog}
	\rev{\mathcal{L}_{\mathrm{PINN}}(\tilde{f}) = \big\|\mathrm{Res}_{pde}(\tilde{f})\big\|_{L^2_{t,v}}^2 + \lambda_{ini} \big\|\mathrm{Res}_{ini}(\tilde{f})\big\|_{L^2_{v}}^2,}
\end{align}
\rev{which is the form used for the counterexamples below.} 

Throughout this section, for simplicity, we assume that $f_0$ satisfies the following normalization:
\begin{align}\label{eq:initial_cond}
\int_{\mathbb{R}^3} f_0(v)\begin{pmatrix}1\\ v\\ |v|^2\end{pmatrix}\,dv
=\begin{pmatrix}1\\ \mathbf{0}\\ 3\end{pmatrix}.
\end{align}
In the spatially homogeneous setting, the macroscopic moments remain constant, which directly implies $\mathcal{M}[f(t,\cdot)]=\mathcal{M}[f_0]$. Consequently, we can obtain the exact solution $f$ with the explicit formula:
$$f(t,v)=e^{-t}f_0(v)+(1-e^{-t})\mathcal{M}_{(1,\mathbf{0},1)}(v)$$
for all $t\in[0,\mathcal{T}]$.

To construct our counterexamples, we first introduce a class of functions, $K_\varepsilon(v)$, parametrized with $\varepsilon>0$, satisfying the following conditions
\begin{enumerate}
    \item $\|K_\varepsilon\|_{L_v^2} = \mathcal{O}(\varepsilon)$.
    \item Its contribution to macroscopic moments (specifically energy) is given as follows:
    \begin{align*}
        \int_{\mathbb{R}^3}\begin{pmatrix}1\\ v\\ |v|^2\end{pmatrix}K_\varepsilon(v)\,dv = \begin{pmatrix} \mathcal{O}(\varepsilon) \\ \mathbf{0} \\ 1 + \mathcal{O}(\varepsilon) \end{pmatrix}.
    \end{align*}
\end{enumerate}

In the rest of this section, we show that when this function $K_\varepsilon$ is added as a perturbation to the initial condition or the PDE source term of~\eqref{eq:homo_bgk}, the resulting $L^2$ PINN loss remains small ($\mathcal{O}(\varepsilon^2)$) due to the first condition of $K_\varepsilon$. However, the second property causes the approximate solution $\tilde{f}_\varepsilon$ to possess macroscopic moments that differ significantly from those of the exact solution. This inherent macroscopic discrepancy imposes a non-negligible lower bound on the actual $L^2$ distance $\|\tilde{f}_\varepsilon - f\|_{L^2}$, which persists regardless of how small $\varepsilon$ becomes.
\begin{Remark}
    A specific example of such a function $K_\varepsilon(v)$ can be constructed using a superposition of two Maxwellians located in the high-velocity region. Let $\mathbf{u}_\varepsilon = (1/\sqrt{\varepsilon}, 0, 0) \in \mathbb{R}^3$. We define:
        \begin{align}\label{K_definition}
        K_\varepsilon(v)=\mathcal{M}_{(\varepsilon/2, \, \mathbf{u}_\varepsilon, \, 1)}(v) + \mathcal{M}_{(\varepsilon/2, \, -\mathbf{u}_\varepsilon, \, 1)}(v).
    \end{align}
    By elementary computations, we can exactly evaluate its macroscopic moments as follows:
    \begin{align} \label{K}
        \int_{\mathbb{R}^3} K_\varepsilon(v) \begin{pmatrix} 1 \\ v \\ |v|^2 \end{pmatrix}dv 
        = \frac{\varepsilon}{2} \begin{pmatrix} 1 \\ \mathbf{u}_\varepsilon \\ |\mathbf{u}_\varepsilon|^2 + 3 \end{pmatrix} 
        + \frac{\varepsilon}{2} \begin{pmatrix} 1 \\ -\mathbf{u}_\varepsilon \\ |-\mathbf{u}_\varepsilon|^2 + 3 \end{pmatrix}
        = \begin{pmatrix} \varepsilon \\ \mathbf{0} \\ 1 + 3\varepsilon \end{pmatrix}.
    \end{align}
    Furthermore, its exact $L^2$ norm can be computed as:
    \begin{align*}
        \|K_\varepsilon\|_{L^2_v}^2 = \int_{\mathbb{R}^3} \left( \mathcal{M}_{(\varepsilon/2, \, \mathbf{u}_\varepsilon, \, 1)} + \mathcal{M}_{(\varepsilon/2, \, -\mathbf{u}_\varepsilon, \, 1)} \right)^2 dv = \frac{\varepsilon^2}{16 \pi^{3/2}} (1 + e^{-1/\varepsilon}) = \mathcal{O}(\varepsilon^2).
    \end{align*}
    Figure~\ref{fig:K_eps_profile} displays the profile of $K_\varepsilon(v)$ for $\varepsilon = 0.01$. While the overall magnitude is extremely small—yielding an $\mathcal{O}(\varepsilon^2)$ squared $L^2$ norm—its strict concentration in the high-velocity region generates a non-vanishing $\mathcal{O}(1)$ macroscopic energy moment.
    \begin{figure}
        \centering
        \includegraphics[width=0.5\linewidth]{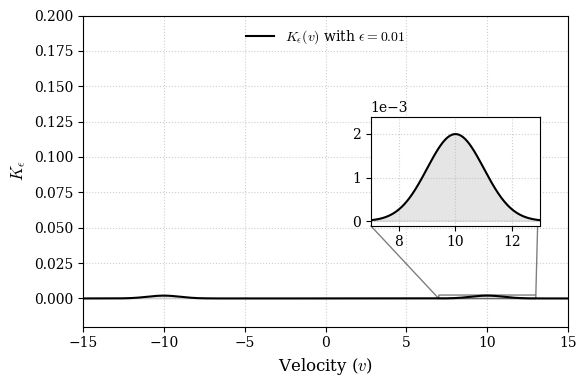}
        \caption{Visualization of the function $K_\varepsilon(v)$ for $\varepsilon = 0.01$.}
        \label{fig:K_eps_profile}
\end{figure}
\end{Remark}

Equipped with $K_\varepsilon(v)$, we now construct two explicit families of counterexamples to demonstrate the inadequacy of the standard $L^2$ loss. The first example considers a scenario with a small initial residual, while the second examines the impact of a small PDE residual.

\subsection{Example 1: Small Initial Residual ($\tilde{f}_\varepsilon^{(1)}$)}

Let us consider the first family of approximate solutions $\{\tilde{f}_\varepsilon^{(1)}\}_{\varepsilon>0}$ satisfying the exact PDE but having a perturbed initial condition:
\begin{align}\label{eq-counter1}
\begin{split}
    \partial_t \tilde{f}_\varepsilon^{(1)} &= \mathcal{M}[\tilde{f}_\varepsilon^{(1)}] - \tilde{f}_\varepsilon^{(1)}, \quad\qquad (t,v) \in (0,\mathcal{T}) \times \mathbb{R}^3, \\ 
    \tilde{f}_\varepsilon^{(1)}(0,v) &= f_0(v)+K_\varepsilon(v), \qquad v \in \mathbb{R}^3. 
    \end{split}
\end{align}
This family of approximate solutions demonstrates that the $L^2$-based standard PINN loss can be made arbitrarily small, but the resulting approximation error does not converge to zero.
\begin{proposition}
     For the family $\{\tilde{f}_\varepsilon^{(1)}\}_{\varepsilon>0}$ of approximate solutions given by~\eqref{eq-counter1}, the following hold:
     \begin{enumerate}
         \item $\mathcal{L}_{\mathrm{PINN}}(\tilde{f}_\varepsilon^{(1)})=\mathcal{O}(\varepsilon^2)$
         \item For the exact solution $f$ to~\eqref{eq:homo_bgk}, there exists a positive constant $C>0$ such that the following holds:
         \begin{align*}
             \| (\tilde{f}_\varepsilon^{(1)} - f)(t) \|_{L^2_v} \geq \frac{C}{2}(1-e^{-t}) > 0,
         \end{align*}
         for any arbitrarily given time $t \in (0, \mathcal{T}]$ and sufficiently small $\varepsilon>0$.
     \end{enumerate}
\end{proposition}
\begin{proof}
The standard PINN loss is purely determined by the initial condition residual:
\begin{align*}
    \mathcal{L}_{\mathrm{PINN}}(\tilde{f}_\varepsilon^{(1)}) &= \left\|\mathrm{Res}_{pde}(\tilde{f}_\varepsilon^{(1)})\right\|_{L^2_{t,v}}^2 + \lambda_{ini}\left\|\mathrm{Res}_{ini}(\tilde{f}_\varepsilon^{(1)})\right\|_{L^2_v}^2\\
    &= \lambda_{ini}\left\|K_\varepsilon\right\|_{L^2_v}^2 = \mathcal{O}(\varepsilon^2).
\end{align*}
We note that $\tilde{f}_\varepsilon^{(1)}$ can be obtained explicitly as:
\begin{align*}
    \tilde{f}_\varepsilon^{(1)}(t,v) = e^{-t} (f_0(v) + K_\varepsilon(v)) + \mathcal{M}[\tilde{f}_\varepsilon^{(1)}](v)(1-e^{-t}). 
\end{align*}
Due to the collision invariance, the macroscopic quantities of $\tilde{f}_\varepsilon^{(1)}$ are constant in time. Using~\eqref{eq:initial_cond} and~\eqref{K}, we compute:
\begin{align*}
    \int_{\mathbb{R}^3} \tilde{f}_\varepsilon^{(1)} \begin{pmatrix} 1 \\ v \\ |v- u_{\tilde{f}_\varepsilon^{(1)}}|^2 \end{pmatrix}dv 
    = \begin{pmatrix} \rho_{\tilde{f}_\varepsilon^{(1)}}  \\ \rho_{\tilde{f}_\varepsilon^{(1)}}u_{\tilde{f}_\varepsilon^{(1)}} \\ 3\rho_{\tilde{f}_\varepsilon^{(1)}} T_{\tilde{f}_\varepsilon^{(1)}} \end{pmatrix}
    = \begin{pmatrix} 1 + \varepsilon    \\ \mathbf{0} \\ 4 + 3\varepsilon \end{pmatrix}.
\end{align*}
Note that the mass and temperature are strictly positive and all moments are uniformly bounded. To show that the $L^2_v$-norm of the difference between $f$ and $\tilde{f}_\varepsilon^{(1)}$ is not small, we explicitly compute the difference:
\begin{align*}
    (\tilde{f}_\varepsilon^{(1)} - f)(t,v) = e^{-t} K_\varepsilon(v)  + (1-e^{-t}) \big( \mathcal{M}[\tilde{f}_\varepsilon^{(1)}](v) - \mathcal{M}[f](v) \big).
\end{align*}
Using the reverse triangle inequality, we obtain the following estimate:
\begin{align*}
    \| (\tilde{f}_\varepsilon^{(1)}-f)(t)\|_{L^2_v} \geq (1-e^{-t}) \| \mathcal{M}[\tilde{f}_\varepsilon^{(1)}] - \mathcal{M}[f] \|_{L^2_v} - e^{-t} \|K_\varepsilon\|_{L^2_v}.
\end{align*}
To find a lower bound for $\| \mathcal{M}[\tilde{f}_\varepsilon^{(1)}] - \mathcal{M}[f] \|_{L^2_v}$, note that as $\varepsilon \rightarrow 0$, the macroscopic quantities of $\tilde{f}_\varepsilon^{(1)}$ converge to $(\rho, u, T) = (1, \mathbf{0}, 4/3)$, whereas the exact solution has moments $(\rho, u, T) = (1, \mathbf{0}, 1)$. Since the two Maxwellians converge to states with fundamentally different temperatures, their $L^2_v$ difference remains strictly bounded away from zero. Hence, by continuity, there exists a constant $C_0>0$ such that for all sufficiently small $\varepsilon>0$,
\begin{align*}
     \| \mathcal{M}[\tilde{f}_\varepsilon^{(1)}] - \mathcal{M}[f] \|_{L^2_v} \geq C_0 > 0.
\end{align*}
Additionally, we have the bound $\|K_\varepsilon\|_{L^2_v} \le C_1 \varepsilon$ for some absolute constant $C_1 > 0$. Substituting these into the estimate yields:
\begin{align*}
    \| (\tilde{f}_\varepsilon^{(1)}-f)(t)\|_{L^2_v} \geq C_0(1-e^{-t}) - C_1 \varepsilon e^{-t}.
\end{align*}
For any arbitrarily given $t \in (0, \mathcal{T}]$, we can choose $\varepsilon>0$ small enough such that $C_1 \varepsilon e^{-t} \leq \frac{C_0}{2}(1-e^{-t})$. Thus, we obtain the uniform lower bound:
\begin{align*}
    \| (\tilde{f}_\varepsilon^{(1)}-f)(t)\|_{L^2_v} \geq \frac{C_0}{2}(1-e^{-t}) > 0.
\end{align*}
This implies that the $L^2_v$ error does not vanish but grows strictly away from zero for any time $t>0$, completing the proof.   
\end{proof}

Figure~\ref{fig:loss_vs_error}(a) plots the actual solution error against the standard PINN loss for varying $\varepsilon$. While the standard PINN loss approaches zero, the distance between the exact solution $f$ and $\tilde{f}_\varepsilon^{(1)}$ exhibits a strictly positive lower bound.

\subsection{Example 2: Small PDE Residual ($\tilde{f}_\varepsilon^{(2)}$)}

In this second example, we construct an approximate solution where the initial condition is perfectly satisfied, but a small PDE residual remains. Let us consider the second family of approximate solutions $\tilde{f}_\varepsilon^{(2)}$ which is defined by: 
\begin{align}\label{eq-counter2}
\begin{split}
    \partial_t \tilde{f}_\varepsilon^{(2)} &= \left(\mathcal{M}[\tilde{f}_\varepsilon^{(2)}] - \tilde{f}_\varepsilon^{(2)}\right)+K_\varepsilon(v), \quad\qquad (t,v) \in (0,\mathcal{T}) \times \mathbb{R}^3, \\ 
    \tilde{f}_\varepsilon^{(2)}(0,v) &= f_0(v), \qquad v \in \mathbb{R}^3. 
    \end{split}
\end{align}
Similar to the first example $\tilde{f}_\varepsilon^{(1)}$, this family of approximate solutions demonstrates that perfectly satisfying the initial condition while driving the $L^2$ PDE residual to zero is still fundamentally insufficient to guarantee the convergence of the approximate solution to the true solution.
\begin{proposition}
     For the family $\{\tilde{f}_\varepsilon^{(2)}\}_{\varepsilon>0}$ of approximate solutions given by~\eqref{eq-counter2}, the following hold:
     \begin{enumerate}
         \item $\mathcal{L}_{\mathrm{PINN}}(\tilde{f}_\varepsilon^{(2)})=\mathcal{O}(\varepsilon^2)$
         \item For the exact solution $f$ to~\eqref{eq:homo_bgk}, there exists a positive constant $C>0$ such that the following holds:
         \begin{align*}
             \| (\tilde{f}_\varepsilon^{(2)} - f)(t) \|_{L^2_v} \geq C (t - 1 + e^{-t}) > 0,
         \end{align*}
         for any arbitrarily given time $t \in (0, \mathcal{T}]$ and sufficiently small $\varepsilon>0$.
     \end{enumerate}
\end{proposition}
\begin{proof}
In this case, the standard PINN loss is purely determined by the PDE residual, yielding:
\begin{align*}
    \mathcal{L}_{\mathrm{PINN}}(\tilde{f}_\varepsilon^{(2)}) = \|K_{\varepsilon}\|_{L^2_{t,v}}^2 = \mathcal{T}\|K_\varepsilon\|_{L^2_v}^2 = \mathcal{O}(\varepsilon^2).
\end{align*}
Unlike the previous example, the macroscopic quantities of $\tilde{f}_\varepsilon^{(2)}$ are time-dependent. Utilizing~\eqref{K}, their rates of change are:
\begin{align*}
    \frac{d}{dt} \int_{\mathbb{R}^3} \tilde{f}_\varepsilon^{(2)} \begin{pmatrix} 1 \\ v \\ |v|^2 \end{pmatrix}dv = \int_{\mathbb{R}^3} K_{\varepsilon}(v) \begin{pmatrix} 1 \\ v \\ |v|^2 \end{pmatrix}dv = \begin{pmatrix} \varepsilon \\ \mathbf{0} \\ 1+3\varepsilon \end{pmatrix}. 
\end{align*}
Integrating over time, the macroscopic quantities of $\tilde{f}_\varepsilon^{(2)}$ explicitly become:
\begin{align*}
    \rho_{\tilde{f}_\varepsilon^{(2)}}(t) = 1 + \varepsilon t, \quad 
    u_{\tilde{f}_\varepsilon^{(2)}}(t) = \mathbf{0}, \quad 
    T_{\tilde{f}_\varepsilon^{(2)}}(t) = \frac{3 + t(1+3\varepsilon)}{3(1+\varepsilon t)}.
\end{align*}
The difference $\tilde{f}_\varepsilon^{(2)} - f$ can be written using Duhamel's principle as:
\begin{align*}
    \tilde{f}_\varepsilon^{(2)} - f = (1-e^{-t})K_{\varepsilon}(v) + \int_0^t \big( \mathcal{M}[\tilde{f}_\varepsilon^{(2)}](s,v)- \mathcal{M}[f](v) \big)e^{-(t-s)}ds.
\end{align*}
To establish a lower bound for the $L^2_v$ error, we apply the projection inequality:
\begin{align*}
    \| (\tilde{f}_\varepsilon^{(2)} - f)(t) \|_{L^2_v} \geq \frac{|\langle (\tilde{f}_\varepsilon^{(2)}-f)(t),\phi\rangle_{L^2_v}|}{\| \phi \|_{L^2_v}}
\end{align*}
for any non-zero test function $\phi(v) \in L^2_v$. Choosing $\phi(v) = e^{-|v|^2}$, we use the exact integration property:
\begin{align*}
    \int_{\mathbb{R}^3} \mathcal{M}_{(\rho, u, T)}(v) e^{-|v|^2} dv = \rho(1+2T)^{-3/2} e^{-\frac{|u|^2}{1+2T}}.
\end{align*}
By applying this identity, the inner product with the perturbation $K_\varepsilon$ can be explicitly evaluated. Since $K_\varepsilon$ consists of two Maxwellians with density $\varepsilon/2$, velocity $\pm \mathbf{u}_\varepsilon$ where $|\mathbf{u}_\varepsilon|^2 = 1/\varepsilon$, and temperature $T=1$, we obtain exactly:
\begin{align*}
    \langle K_{\varepsilon}, \phi \rangle_{L^2_v} =  3^{-3/2} \varepsilon e^{-\frac{1}{3\varepsilon}}.
\end{align*}
This reveals that the projection of the initial perturbation decays exponentially with respect to $1/\varepsilon$, making it extremely negligible for small $\varepsilon$.

Now, taking the inner product of $\tilde{f}_\varepsilon^{(2)} - f$ with $\phi$ yields:
\begin{align*}
    \langle (\tilde{f}_\varepsilon^{(2)}-f)(t), \phi\rangle_{L^2_v} 
    &= - \int_0^t e^{-(t-s)} I_\varepsilon(s) ds + (1-e^{-t}) 3^{-3/2} \varepsilon e^{-\frac{1}{3\varepsilon}},
\end{align*}
where $I_\varepsilon(s) := 3^{-3/2} - \rho_{\tilde{f}_\varepsilon^{(2)}}(s)(1+2T_{\tilde{f}_\varepsilon^{(2)}}(s))^{-3/2}$. Taking the absolute value and applying the reverse triangle inequality gives:
\begin{align*}
    |\langle (\tilde{f}_\varepsilon^{(2)}-f)(t),\phi\rangle_{L^2_v}| \geq \int_0^t e^{-(t-s)} I_\varepsilon(s) ds - (1-e^{-t}) 3^{-3/2} \varepsilon e^{-\frac{1}{3\varepsilon}}.
\end{align*}
As $\varepsilon \rightarrow 0$, the macroscopic moments converge as $\rho_{\tilde{f}_\varepsilon^{(2)}}(s) \rightarrow 1$ and $T_{\tilde{f}_\varepsilon^{(2)}}(s) \rightarrow 1+s/3$. Thus, $I_\varepsilon$ converges to the limit function $I_0(s) := 3^{-3/2} - (3+2s/3)^{-3/2}$ \rev{in $C^1([0, \mathcal{T}])$, since the moments above are smooth in $(s,\varepsilon)$}. 

By the Mean Value Theorem, since $\frac{d}{ds} (3+2s/3)^{-3/2} = -(3+2s/3)^{-5/2}$, we can bound $I_0(s)$ from below linearly for $s \in [0, \mathcal{T}]$:
\begin{align*}
    I_0(s) \geq \kappa s, \quad \text{where} \quad \kappa := (3+2\mathcal{T}/3)^{-5/2} > 0.
\end{align*}
\rev{Due to the $C^1$ convergence of $I_\varepsilon \to I_0$ and $I_\varepsilon(0) = 0$, for sufficiently small $\varepsilon>0$ we have $I_\varepsilon'(s) \geq \frac{\kappa}{2}$ on $[0,\mathcal{T}]$, and hence $I_\varepsilon(s) \geq \frac{\kappa}{2} s$.} Substituting this into the time integral yields:
\begin{align*}
    \int_0^t e^{-(t-s)} I_\varepsilon(s) ds \geq \frac{\kappa}{2} \int_0^t e^{-(t-s)} s ds = \frac{\kappa}{2} (t - 1 + e^{-t}).
\end{align*}
Notice that $(t - 1 + e^{-t}) > 0$ for all $t > 0$. Since the negative perturbation term decays exponentially ($e^{-1/(3\varepsilon)}$), for any arbitrarily given $t \in (0, \mathcal{T}]$, we can choose $\varepsilon>0$ small enough such that the negative term is strictly bounded by $\frac{\kappa}{4} (t - 1 + e^{-t})$. 
Therefore, we arrive at the uniform lower bound for the $L^2_v$ norm for any arbitrary $t \in (0, \mathcal{T}]$:
\begin{align*}
    \| (\tilde{f}_\varepsilon^{(2)} - f)(t) \|_{L^2_v} \geq \frac{1}{\| \phi \|_{L^2_v}} \left( \frac{\kappa}{4} (t - 1 + e^{-t}) \right) := C_0 (t - 1 + e^{-t}) > 0.
\end{align*}
This shows that the error grows strictly away from zero for any time $t>0$, demonstrating that a vanishing PDE residual does not guarantee a small solution error.
\end{proof}

Similarly, Figure~\ref{fig:loss_vs_error}(b) illustrates the corresponding behavior for the standard $L^2$ PINN loss. Even as the standard loss becomes vanishingly small, the error $\sup_{t \in (0,\mathcal{T}]}\|(f-\tilde{f}_\varepsilon^{(2)})(t)\|_{L^2_v}$ fails to decay.

\begin{figure}[htbp]
    \centering
    \begin{subfigure}[b]{0.48\textwidth}
        \centering
        \includegraphics[width=\textwidth]{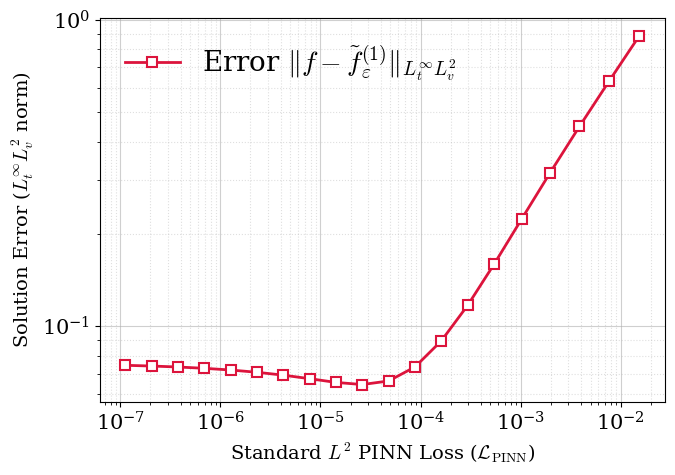}
        \caption{}
        \label{fig:example1_error}
    \end{subfigure}
    \hfill
    \begin{subfigure}[b]{0.48\textwidth}
        \centering
        \includegraphics[width=\textwidth]{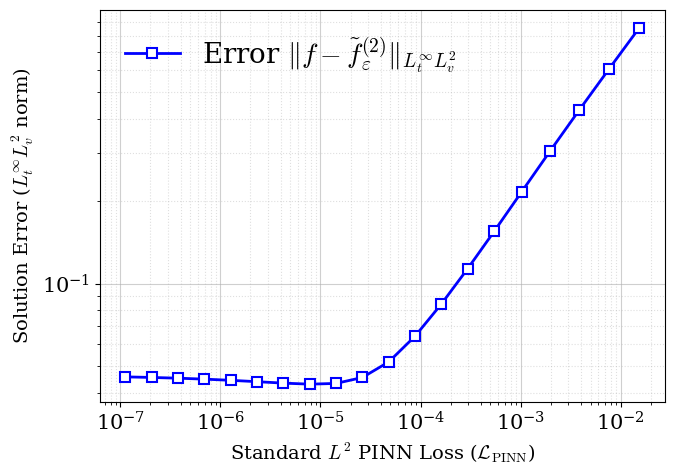}                
        \caption{}
        \label{fig:example2_error}
    \end{subfigure}
    \caption{Loss-accuracy curves for the explicit counterexamples: (a) Example 1 and (b) Example 2.}
    \label{fig:loss_vs_error}
\end{figure}

As demonstrated in the two examples above, these counterexamples are not pathological mathematical artifacts. Neither example exhibits singular behaviors, such as blowing up or vanishing within a finite time. The macroscopic moments (mass, momentum, and energy) of these approximate solutions remain strictly positive and uniformly bounded over time. This ensures that the counterexamples reside within the standard function spaces conventionally used to establish the well-posedness of the BGK model.

\rev{We also note that these counterexamples were constructed to illustrate a potential failure case of PINNs based on the standard $L^2$ loss. However, Subsection~\ref{subsec:homogeneous} confirms that this failure is not merely a hypothetical phenomenon but actually occurs in practice.}

\section{Weighted PINN Loss and Stability Analysis}
\label{sec:stability}

As demonstrated in Section~\ref{sec:counterexamples}, the standard $L^2$ PINN loss is fundamentally inadequate for solving the BGK model. The explicit counterexamples reveal a critical vulnerability: an approximation can achieve a vanishingly small standard loss ($\mathcal{L}_{\mathrm{PINN}} \to 0$) while suffering from a non-vanishing error with respect to the true solution. The key insight drawn from these results is that small, highly concentrated errors in the high-velocity region can corrupt the macroscopic energy moment. Therefore, it is essential to rigorously control the velocity moments of the residual errors at least up to the second order (i.e., bounding $\int_{\mathbb{R}^3} |v|^2 \,|\mathrm{Res}_{pde}(\tilde{f})| \,dv$).

To achieve this, we naturally consider introducing a velocity-dependent weight $w(v)$ into the standard framework. For instance, the loss function for the PDE residual can be reformulated as the squared weighted $L^2$ norm:
\begin{align*}
    \mathcal{L}_{w,pde}(\tilde{f}) := \int_0^\mathcal{T} \int_{\mathbb{R}^3} \int_{[0,1]^3} \left| w(v) \left( \partial_t \tilde{f} + v \cdot \nabla_x \tilde{f} -\frac{1}{\mathrm{Kn}} \big( \mathcal{M}[\tilde{f}] - \tilde{f} \big) \right) \right|^2 dx dv dt.
\end{align*}
Similarly, this reformulation can be straightforwardly applied to the initial and boundary condition residuals. By appropriately choosing the weight function (e.g., $w(v)=1+\alpha|v|^\beta$), one can sufficiently control the second-order moments of the PDE and IC residuals, thereby preventing our approximations from converging to the specific incorrect solutions presented in Section~\ref{sec:counterexamples}. However, successfully ruling out these particular counterexamples does not inherently guarantee that this empirically motivated modification is sufficient to ensure the overall reliability of the approximate solution in general.

\rev{In this section, we address this issue by rigorously establishing a stability theory for the weighted PINN loss applied to the BGK model. Specifically, we prove that under suitable integrability conditions on the weight $w(v)$, minimizing the weighted residual ensures the convergence of the approximate solution to the exact solution. Throughout this section, we set $\mathrm{Kn}=1$ for notational simplicity (the same argument works for any fixed $\mathrm{Kn}>0$), and $f(t,x,v)$ denotes the exact solution of the BGK problem~\eqref{eq:bgk}.
	
We first define the function space in which the exact and approximate solutions reside. Instead of seeking stability in an arbitrarily large or irregular space, we naturally restrict our attention to a physically meaningful class of velocity distributions.}

\begin{definition}[Ansatz Space] \label{def:ansatz_space}
Let $M=(C_{\rho, \min},$ $C_{\rho, \max},$ $C_{u, \max},$ $C_{T, \min},$ $C_{T, \max},$ $C_\infty)$ denote a generic set of positive structural constants. We define the ansatz space $\mathcal{X}_{M}$ on $[0,\mathcal{T}] \times [0,1]^3 \times \mathbb{R}^3$ as the set of non-negative functions
\begin{align*}
    f \in C^1\left([0,\mathcal{T}]\times[0,1]^3;L_2^1(\mathbb{R}^3_v)\cap L_v^2\right)
\end{align*}
\rev{where $L_2^1(\mathbb{R}^3_v) := \{ g : \int_{\mathbb{R}^3} (1+|v|^2)\,|g(v)|\, dv < \infty \}$,} satisfying the following conditions:
\begin{itemize}

    \item \textbf{Uniform boundedness of macroscopic moments:} There exist positive constants such that for all $(t,x)$,
    \begin{align*}
        0 < C_{\rho, \min} \leq \rho_f(t,x) \leq C_{\rho, \max},\\
        0 < C_{T, \min} \leq T_f(t,x) \leq C_{T, \max},\\
        |u_f(t,x)| \leq C_{u, \max}.
    \end{align*}
    \item \textbf{Uniform boundedness of the $L^2_v$ norm:} The $L^2_v$ norm is uniformly bounded:  
    \begin{align*} 
        \underset{t,x}{\text{ess}\sup}\|f\|_{L^2_{v}}\leq C_\infty.
    \end{align*}
\end{itemize}
\end{definition}

\begin{Remark}[On the exact solution]
For initial data $f_0$ being sufficiently regular and having sufficiently fast decaying rates in velocity domain, the exact solution $f$ of the BGK model belongs to $\mathcal{X}_{M}$ for some structural constants $M$. More details can be found in~\cite{yun2015ellipsoidal}.
\end{Remark}

\begin{Remark}[On the approximate solution]
It is reasonable to assume that the approximate solution $\tilde{f}$ generated by a neural network also belongs to $\mathcal{X}_{M}$. In practice, recent methodologies for solving the BGK model via physics-informed neural networks explicitly structure the network ansatz to enforce this behavior. Specifically, the approximate solution is often parameterized using a micro-macro decomposition form:
$$\tilde{f}(t,x,v) = \mathcal{M}_{(\tilde{\rho}_{eq}(t,x),\, \tilde{u}_{eq}(t,x),\, \tilde{T}_{eq}(t,x))}(v) + e^{-\tau^2|v-\mu|^2/2}\,\tilde{f}^{neq}(t,x,v),$$
where $\tilde{\rho}_{eq}, \tilde{u}_{eq}, \tilde{T}_{eq}$ and $\tilde{f}^{neq}$ are outputs of the neural networks, and $\mu \in \mathbb{R}^3$ and $\tau > 0$ are trainable localization parameters. Assuming that the macroscopic outputs remain bounded (i.e., do not blow up) during training and $\tilde{f}^{neq}$ exhibits at most polynomial growth with respect to the velocity variable, the exponential weight strictly dominates. As a result, the approximate solution naturally exhibits rapid decay in the velocity domain, thereby belonging to $\mathcal{X}_{M}$ for some structural constants $M$.
\end{Remark}

With the ansatz space $\mathcal{X}_{M}$ defined, we now recall a Lipschitz-type continuity of the local Maxwellian operator within this space. This property is crucial for establishing our desirable stability estimate.

\begin{lemma}[\cite{yun2015ellipsoidal}, Lipschitz-type continuity of the local Maxwellian] \label{lem:M_lipschitz}
Suppose that $f$ and $\tilde{f}$ belong to the admissible ansatz space $\mathcal{X}_{M}$. Then, there exist positive constants $L_M$ and $c_M$, depending only on the structural constants of $M$, such that for all $(t,x,v) \in [0,\mathcal{T}] \times [0,1]^3 \times \mathbb{R}^3$, we have:
\begin{align*}
    |\mathcal{M}[f](t,x,v) - \mathcal{M}[\tilde{f}](t,x,v)| \leq L_M e^{-c_M|v|^2} \int_{\mathbb{R}^3} (1+|v_*|^2)|f(t,x,v_*) - \tilde{f}(t,x,v_*)| dv_*.
\end{align*}
\end{lemma}
\begin{proof}
By the definition of the ansatz space $\mathcal{X}_{M}$, the macroscopic quantities $(\rho_f, u_f, T_f)$ and $(\rho_{\tilde{f}}, u_{\tilde{f}}, T_{\tilde{f}})$ are strictly positive and uniformly bounded. Since the mapping $(\rho, u, T) \mapsto \mathcal{M}_{(\rho, u, T)}(v)$ is continuously differentiable with respect to its parameters, applying the mean value theorem yields that there exist positive constants $C_1$ and $c_M$ (depending only on the bounds in $M$) such that
\begin{align}\label{eq:M_lipschitz_intermediate}
    |\mathcal{M}[f] - \mathcal{M}[\tilde{f}]| \leq C_1 e^{-c_M|v|^2} \Big( |\rho_f - \rho_{\tilde{f}}| + |u_f - u_{\tilde{f}}| + |T_f - T_{\tilde{f}}| \Big).
\end{align}
Next, we bound the differences of the macroscopic quantities in terms of the distribution functions. For the density, we readily have:
\begin{align}\label{eq:bound_rho}
    |\rho_f - \rho_{\tilde{f}}| = \left| \int_{\mathbb{R}^3} (f - \tilde{f}) dv \right| \leq \int_{\mathbb{R}^3} |f - \tilde{f}| dv \leq \int_{\mathbb{R}^3} (1+|v|^2)|f - \tilde{f}| dv.
\end{align}
For the macroscopic velocity, we algebraically expand the difference as follows:
\begin{align}\label{eq:bound_u}
\begin{split}
    |u_f - u_{\tilde{f}}| &= \left| \frac{1}{\rho_f}\int_{\mathbb{R}^3} v f dv - \frac{1}{\rho_{\tilde{f}}}\int_{\mathbb{R}^3} v \tilde{f} dv \right| \\
    &= \left| \frac{1}{\rho_f}\int_{\mathbb{R}^3} v (f - \tilde{f}) dv - \frac{u_{\tilde{f}}}{\rho_f}(\rho_f - \rho_{\tilde{f}}) \right| \\
    &\leq \frac{1}{C_{\rho, \min}} \int_{\mathbb{R}^3} |v| |f - \tilde{f}| dv + \frac{C_{u, \max}}{C_{\rho, \min}} \int_{\mathbb{R}^3} |f - \tilde{f}| dv \\
    &\leq C_2 \int_{\mathbb{R}^3} (1+|v|^2)|f - \tilde{f}| dv,
\end{split}
\end{align}
for some constant $C_2 > 0$ depending on $M$. Similarly, for the temperature, recalling the energy relation $3 \rho_f T_f = \int_{\mathbb{R}^3} |v|^2 f dv - \rho_f |u_f|^2$ in velocity space, we obtain:
\begin{align*}
    |T_f - T_{\tilde{f}}| &= \left| \left( \frac{1}{3 \rho_f}\int_{\mathbb{R}^3} |v|^2 f dv - \frac{|u_f|^2}{3} \right) - \left( \frac{1}{3 \rho_{\tilde{f}}}\int_{\mathbb{R}^3} |v|^2 \tilde{f} dv - \frac{|u_{\tilde{f}}|^2}{3} \right) \right| \\
    &\leq \frac{1}{3} \left| \frac{1}{\rho_f}\int_{\mathbb{R}^3} |v|^2 (f-\tilde{f}) dv - \frac{\rho_f-\rho_{\tilde{f}}}{\rho_f \rho_{\tilde{f}}} \int_{\mathbb{R}^3} |v|^2 \tilde{f} dv \right| + \frac{1}{3} \left| |u_f|^2 - |u_{\tilde{f}}|^2 \right|.
\end{align*}
Since $\tilde{f} \in \mathcal{X}_{M}$, its macroscopic energy is uniformly bounded by some constant $C_E > 0$. Using the identity $|u_f|^2 - |u_{\tilde{f}}|^2 = (u_f - u_{\tilde{f}}) \cdot (u_f + u_{\tilde{f}})$, we further bound the temperature difference:
\begin{align}\label{eq:bound_T}
\begin{split}
    |T_f - T_{\tilde{f}}| &\leq \frac{1}{3 C_{\rho, \min}} \int_{\mathbb{R}^3} |v|^2 |f-\tilde{f}| dv + \frac{C_E}{3 C_{\rho, \min}^2} \int_{\mathbb{R}^3} |f-\tilde{f}| dv + \frac{2C_{u, \max}}{3} |u_f - u_{\tilde{f}}| \\
    &\leq C_3 \int_{\mathbb{R}^3} (1+|v|^2)|f-\tilde{f}| dv,
\end{split}
\end{align}
where $C_3 > 0$ depends on $M$ via~\eqref{eq:bound_rho} and~\eqref{eq:bound_u}. Finally, substituting the bounds~\eqref{eq:bound_rho},~\eqref{eq:bound_u}, and~\eqref{eq:bound_T} back into~\eqref{eq:M_lipschitz_intermediate}, we conclude that there exists a constant $L_M > 0$ such that
\begin{align*}
    |\mathcal{M}[f](t,x,v) - \mathcal{M}[\tilde{f}](t,x,v)| \leq L_M e^{-c_M|v|^2} \int_{\mathbb{R}^3} (1+|v_*|^2)|f(t,x,v_*) - \tilde{f}(t,x,v_*)| dv_*,
\end{align*}
which completes the proof.
\end{proof}

Equipped with the Lipschitz property of the Maxwellian, we are now ready to establish our main stability estimate. To achieve this, we introduce a general velocity-dependent weight function $w(v) \geq 1$. With this in hand, we present the following theorem.

\begin{theorem}[Stability Estimate] \label{thm:main_stability}
Let $f$ be the exact solution to~\eqref{eq:bgk} and $\tilde{f}$ be an approximate solution. Assume that $f$ and $\tilde{f}$ belong to $\mathcal{X}_{M}$ for some structural constants $M$. Suppose that the weight function $w(v)\geq1$ satisfies the following integrability conditions:
\begin{align} \label{eq:weight_cond}
    \int_{\mathbb{R}^3} \frac{(1+|v|^2)^2}{w(v)^2} dv < \infty, \quad \text{and} \quad \int_{\mathbb{R}^3} w(v)^2 e^{-2c_M|v|^2} dv < \infty,
\end{align}
where $c_M > 0$ is the constant given in Lemma~\ref{lem:M_lipschitz}. Then, we have the following stability estimate for all $t \in [0,\mathcal{T}]$:
\begin{align*}
    \|w(f-\tilde{f})(t)\|_{L^2_{x,v}}^2 \leq{}& C_* \bigg( \|w \, \mathrm{Res}_{ini}(\tilde{f})\|_{L^2_{x,v}}^2 + \int_0^t \| w \, \mathrm{Res}_{pde}(\tilde{f})(s) \|_{L^2_{x,v}}^2 ds\\
    &\qquad + \sum_{i=1}^3 \Big( \int_0^t \|v_i w^2 \, \mathrm{Res}_{bc,i}(\tilde{f})(s)\|_{L_{\hat{x}_i,v}^2}^2 ds \Big)^{1/2} \bigg),
\end{align*}
where $C_* > 0$ is a constant depending on $\mathcal{T}$, $M$, and $w$.
\end{theorem}
\begin{proof}
Let us define the weighted error function as $e(t,x,v) := w(v)(f(t,x,v) - \tilde{f}(t,x,v))$. By the definition of the residuals~\eqref{eq:residuals}, the approximate solution $\tilde{f}$ satisfies the following perturbed BGK model:
\begin{align} 
    \partial_t \tilde{f} + v\cdot \nabla_x \tilde{f} &= \mathcal{M}[\tilde{f}] - \tilde{f} + \mathrm{Res}_{pde}(\tilde{f}), \quad \textrm{in} \; (0,\mathcal{T}) \times [0,1]^3 \times \mathbb{R}^3, \label{eq:BGK_approx} \\
    \tilde{f}(0,x,v) &= f_0(x,v) + \mathrm{Res}_{ini}(\tilde{f}), \quad \textrm{on}\; [0,1]^3 \times \mathbb{R}^3.\label{eq:BGK_approx_initial}
\end{align}
Multiplying the governing equations~\eqref{eq:bgk} and~\eqref{eq:BGK_approx} by $w(v)$ and subtracting them yields the error dynamics:
\begin{align}\label{eq:error_dynamics}
    \partial_t e + v \cdot \nabla_x e = w(v)\big(\mathcal{M}[f] - \mathcal{M}[\tilde{f}]\big) - e - w(v)\mathrm{Res}_{pde}(\tilde{f}).
\end{align}
Multiplying~\eqref{eq:error_dynamics} by $e$ and integrating over the phase space $(x,v) \in [0,1]^3 \times \mathbb{R}^3$, we obtain:
\begin{align}\label{eq:energy_identity}
    \frac{1}{2} \frac{d}{dt} \|e(t)\|_{L_{x,v}^2}^2 + \|e(t)\|_{L_{x,v}^2}^2 ={}& \iint w\big(\mathcal{M}[f] - \mathcal{M}[\tilde{f}]\big) e \,dx dv - \iint (v \cdot \nabla_x e) e \,dx dv - \iint w \mathrm{Res}_{pde}(\tilde{f}) e \,dx dv.
\end{align}
We proceed to estimate the three terms on the right-hand side. For the transport term, applying the divergence theorem and factoring the difference of squares gives:
\begin{align*}
    -\iint (v \cdot \nabla_x e) e \,dx dv &= - \frac{1}{2} \sum_{i=1}^3 \int_{\mathbb{R}^3} \int_{[0,1]^2} v_i \Big( \big[w(v) (f-\tilde{f})|_{x_i=1}\big]^2 - \big[w(v) (f-\tilde{f})|_{x_i=0}\big]^2 \Big) d\hat{x}_i dv \\
    &= \frac{1}{2} \sum_{i=1}^3 \int_{[0,1]^2} \int_{\mathbb{R}^3} \big( v_i w(v)^2 \mathrm{Res}_{bc,i}(\tilde{f}) \big)\big( (f-\tilde{f})|_{x_i=1} + (f-\tilde{f})|_{x_i=0} \big) dv d\hat{x}_i.
\end{align*}
Applying the Cauchy-Schwarz inequality over the velocity integral yields:
\begin{align*}
    -\iint (v \cdot \nabla_x e) e \,dx dv &\leq \frac{1}{2} \sum_{i=1}^3 \int_{[0,1]^2} \|v_i w^2 \mathrm{Res}_{bc,i}(\tilde{f})\|_{L^2_v} \|(f-\tilde{f})|_{x_i=1} + (f-\tilde{f})|_{x_i=0}\|_{L^2_v} d\hat{x}_i \\
    &\leq \underset{t,x}{\text{ess}\sup}\| (f-\tilde{f})(t,x,\cdot) \|_{L^2_v} \sum_{i=1}^3 \int_{[0,1]^2} \|v_i w^2 \mathrm{Res}_{bc,i}(\tilde{f})\|_{L^2_v} d\hat{x}_i\\
    &\leq \underset{t,x}{\text{ess}\sup}\| (f-\tilde{f})(t,x,\cdot) \|_{L^2_v} \sum_{i=1}^3 \|v_i w^2 \mathrm{Res}_{bc,i}(\tilde{f})\|_{L^2_{\hat{x}_i,v}}
\end{align*}
Since $f, \tilde{f} \in \mathcal{X}_{M}$, we have
$$\underset{t,x}{\text{ess}\sup}\| f-\tilde{f} \|_{L_v^2} < 2C_\infty.$$ 
Substituting this back, the boundary term is bounded by:
\begin{align*}
    -\iint (v \cdot \nabla_x e) e \,dx dv \leq 2C_\infty \sum_{i=1}^3 \|v_i w^2 \mathrm{Res}_{bc,i}(\tilde{f})\|_{L^2_{\hat{x}_i,v}}.
\end{align*}

Then, we apply Lemma~\ref{lem:M_lipschitz} to obtain:
\begin{align*}
    &\iint w \big|\mathcal{M}[f] - \mathcal{M}[\tilde{f}]\big| |e| \,dx dv\\
    &\leq \iint w(v) L_M e^{-c_M|v|^2} \left( \int_{\mathbb{R}^3} (1+|v_*|^2) |f - \tilde{f}| dv_* \right) |e(t,x,v)| dx dv\\
    &=\iint w(v) L_M e^{-c_M|v|^2} \left( \int_{\mathbb{R}^3} \frac{1+|v_*|^2}{w(v_*)} |e| dv_* \right)|e(t,x,v)| dx dv.
\end{align*}
Due to the first integrability condition in~\eqref{eq:weight_cond}, the following constant is strictly finite:
\begin{align*}
    C_w := \int_{\mathbb{R}^3} \frac{(1+|v_*|^2)^2}{w(v_*)^2} dv_* < \infty.
\end{align*}
\rev{Using} this and applying the Cauchy-Schwarz inequality, the inner integral is bounded by:
\begin{align*}
    \int_{\mathbb{R}^3} \frac{1+|v_*|^2}{w(v_*)} |e(t,x,v_*)| dv_* \leq \left( \int_{\mathbb{R}^3} \frac{(1+|v_*|^2)^2}{w(v_*)^2} dv_* \right)^{1/2} \|e(t,x,\cdot)\|_{L^2_v} = C_w^{1/2} \|e(t,x,\cdot)\|_{L^2_v}.
\end{align*}
Substituting this back and applying the Cauchy-Schwarz inequality once more over the $v$-integral yields:
\begin{align*}
    \iint w \big|\mathcal{M}[f] - \mathcal{M}[\tilde{f}]\big| |e| \,dx dv &\leq L_M C_w^{1/2} \int_{[0,1]^3} \|e\|_{L^2_v} \left( \int_{\mathbb{R}^3} w(v) e^{-c_M|v|^2} |e(t,x,v)| dv \right) dx \\
    &\leq A \int_{[0,1]^3} \|e\|_{L^2_v}^2 dx = A \|e\|_{L_{x,v}^2}^2,
\end{align*}
where the constant $A$ is defined by:
\begin{align*}
    A := L_M C_w^{1/2} \left( \int_{\mathbb{R}^3} w(v)^2 e^{-2c_M|v|^2} dv \right)^{1/2}.
\end{align*}
Note that $A$ is strictly finite due to the second integrability condition in~\eqref{eq:weight_cond}. Lastly, using Young's inequality for the PDE residual term, we have:
\begin{align*}
    -\iint w \mathrm{Res}_{pde}(\tilde{f}) e \,dx dv \leq \frac{1}{2} \|w \mathrm{Res}_{pde}(\tilde{f})\|_{L_{x,v}^2}^2 + \frac{1}{2} \|e\|_{L_{x,v}^2}^2.
\end{align*}

Collecting all estimates and substituting them into the energy identity~\eqref{eq:energy_identity}, we arrive at the differential inequality:
\begin{align*}
    \frac{d}{dt} \|e\|_{L_{x,v}^2}^2 \leq (2A - 1) \|e\|_{L_{x,v}^2}^2 + \|w \mathrm{Res}_{pde}(\tilde{f})\|_{L_{x,v}^2}^2 +  \rev{4}C_\infty \sum_{i=1}^3 \|v_i w^2 \mathrm{Res}_{bc,i}(\tilde{f})\|_{L^2_{\hat{x}_i,v}}.
\end{align*}
Applying Gr\"{o}nwall's inequality from $0$ to $t$ gives:
\begin{align*}
    \|e(t)\|_{L_{x,v}^2}^2 \leq{}& e^{(2A-1)t} \|e(0)\|_{L_{x,v}^2}^2 \\
    &+ \int_0^t e^{(2A-1)(t-s)} \left( \|w \mathrm{Res}_{pde}(\tilde{f})(s)\|_{L_{x,v}^2}^2 +  \rev{4}C_\infty \sum_{i=1}^3 \|v_i w^2 \mathrm{Res}_{bc,i}(\tilde{f})(s)\|_{L^2_{\hat{x}_i,v}} \right) ds.
\end{align*}
Noting that $\|e(0)\|_{L_{x,v}^2}^2 = \|w \mathrm{Res}_{ini}(\tilde{f})\|_{L_{x,v}^2}^2$, and applying the Cauchy-Schwarz inequality to the time integral of the boundary residual
\begin{align*}
    \int_0^t 1 \cdot \|v_i w^2 \mathrm{Res}_{bc,i}(\tilde{f})(s)\|_{L^2_{\hat{x}_i,v}} ds \leq \sqrt{\mathcal{T}} \left( \int_0^t \|v_i w^2 \mathrm{Res}_{bc,i}(\tilde{f})(s)\|_{L^2_{\hat{x}_i,v}}^2 ds \right)^{1/2},
\end{align*}
we obtain the desired bound:
\begin{align*}
    \|e(t)\|_{L_{x,v}^2}^2 \leq{}& C_* \bigg( \|w \mathrm{Res}_{ini}(\tilde{f})\|_{L_{x,v}^2}^2 + \int_0^t \|w \mathrm{Res}_{pde}(\tilde{f})(s)\|_{L_{x,v}^2}^2 ds+ \sum_{i=1}^3 \Big( \int_0^t \|v_i w^2 \mathrm{Res}_{bc,i}(\tilde{f})(s)\|_{L^2_{\hat{x}_i,v}}^2 ds \Big)^{1/2} \bigg),
\end{align*}
where $C_* = \max(1, \rev{4}C_\infty\sqrt{\mathcal{T}}) e^{|2A-1|\mathcal{T}}$. This completes the proof.
\end{proof}

Motivated by the stability estimate in Theorem~\ref{thm:main_stability}, we propose the theory-guided weighted PINN loss function, denoted as $\mathcal{L}_{w\text{-PINN}}$, incorporating a weight function $w(v)$ satisfying~\eqref{eq:weight_cond} (for instance, $w(v) = 1 + \alpha |v|^\beta$ with $\alpha>0$ and $\beta>7/2$):
\begin{equation}\label{eq:w_pinn_loss_full}
    \mathcal{L}_{w\text{-PINN}}(\tilde{f}) = \mathcal{L}_{w,pde}(\tilde{f}) + \lambda_{bc}\mathcal{L}_{w,bc}(\tilde{f}) + \lambda_{ini} \mathcal{L}_{w,ini}(\tilde{f}),
\end{equation}
\rev{where each weighted loss component is the squared weighted $L^2$ norm of the corresponding residual~\eqref{eq:residuals}:}
\begin{equation}\label{eq:w_L2_components}
\rev{\begin{split}
    \mathcal{L}_{w,pde}(\tilde{f}) &:= \big\| w\, \mathrm{Res}_{pde}(\tilde{f}) \big\|_{L^2_{t,x,v}}^2, \\
    \mathcal{L}_{w,bc}(\tilde{f}) &:= \sum_{i=1}^3 \big\| v_i\, w^2\, \mathrm{Res}_{bc,i}(\tilde{f}) \big\|_{L^2_{t,\hat{x}_i,v}}^2, \\
    \mathcal{L}_{w,ini}(\tilde{f}) &:= \big\| w\, \mathrm{Res}_{ini}(\tilde{f}) \big\|_{L^2_{x,v}}^2.
\end{split}}
\end{equation}
Theorem~\ref{thm:main_stability} directly implies that, unlike the standard $L^2$ PINN loss, minimizing this newly defined $\mathcal{L}_{w\text{-PINN}}$ strictly guarantees the convergence of the approximate solution to the exact solution.

\begin{corollary}[Convergence of the Approximate Solution] \label{cor:solution_convergence}
Let the exact solution $f$ and the approximate solution $\tilde{f}$ belong to the  ansatz space $\mathcal{X}_{M}$, satisfying~\eqref{eq:bgk} and~\eqref{eq:BGK_approx}, respectively. Then, for any weight function $w(v) \geq 1$ satisfying the integrability conditions~\eqref{eq:weight_cond}, we have as $\mathcal{L}_{w\text{-PINN}}(\tilde{f}) \to 0$:
\begin{align*}
    \|w(f-\tilde{f})(t)\|_{L_{x,v}^2} \to 0, \quad \forall t \in [0,\mathcal{T}],
\end{align*}
and furthermore, for all $1\leq p\leq2$
\begin{align*}
    \|(f-\tilde{f})(t)\|_{L_{x,v}^p} \to 0, \quad \forall t \in [0,\mathcal{T}].
\end{align*}
\end{corollary}
\begin{proof}
The $L^2$ convergence ($p=2$) is a direct consequence of Theorem~\ref{thm:main_stability} and the fact that $w(v)\geq1$. For $1 \leq p < 2$, applying H\"{o}lder's inequality yields
\begin{align*}
    \|(f-\tilde{f})(t)\|_p^p &= \int_{[0,1]^3} \int_{\mathbb{R}^3} \frac{1}{w(v)^p} |w(v)(f-\tilde{f})|^p dv\,dx\\
    &\leq \left( \int_{[0,1]^3} \int_{\mathbb{R}^3} |w(v)(f-\tilde{f})|^2 dv\,dx \right)^{\frac{p}{2}} \left( \int_{[0,1]^3} \int_{\mathbb{R}^3} w(v)^{-\frac{2p}{2-p}} dv\,dx \right)^{\frac{2-p}{2}}.
\end{align*}
Since $w(v)\geq1$, we have $w(v)^{-\frac{2p}{2-p}} \leq w(v)^{-2}$. The integrability condition~\eqref{eq:weight_cond} ensures $\int_{[0,1]^3}\int_{\mathbb{R}^3} w(v)^{-2} dv\,dx = \int_{\mathbb{R}^3} w(v)^{-2} dv < \infty$. Therefore, the $L^p$ error is strictly bounded by the weighted $L^2$ error, completing the proof.
\end{proof}

\begin{Remark}[On the counterexamples]
\rev{We emphasize that the explicit counterexamples $\{\tilde{f}_\varepsilon^{(1)}\}_{\varepsilon>0}$ and $\{\tilde{f}_\varepsilon^{(2)}\}_{\varepsilon>0}$ constructed in Section~\ref{sec:counterexamples} belong to $\mathcal{X}_{M}$. Indeed, their macroscopic moments are uniformly bounded, with mass and temperature strictly positive. Moreover, each counterexample is an explicit combination of the initial datum $f_0$, local Maxwellians with these bounded moments, and the small perturbation $K_\varepsilon$ \eqref{K_definition}, so its $L^2_v$ norm is uniformly bounded as well. Consequently, since these counterexamples satisfy the hypotheses of Theorem~\ref{thm:main_stability}, their non-vanishing error forces the weighted PINN loss $\mathcal{L}_{w\text{-PINN}}$ to remain bounded away from zero, even though their standard $L^2$ PINN loss vanishes as $\mathcal{O}(\varepsilon^2)$. In fact, the weighted loss even diverges to infinity, correctly signalling that the counterexamples are not accurate approximations of the exact solution.}

To rigorously prove this for any weight function $w(v)$ satisfying the integrability condition~\eqref{eq:weight_cond}, we can directly establish a lower bound for the weighted $L^2$ norm of $K_\varepsilon(v)$. Let us define a local ball $\Omega_\varepsilon = \{ v \in \mathbb{R}^3 : |v - \mathbf{u}_{\varepsilon}| < 1 \}$ centered at \rev{one of the two symmetric peaks, $+\mathbf{u}_\varepsilon$, of the perturbation}, where $|\mathbf{u}_{\varepsilon}| = \varepsilon^{-1/2}$. For sufficiently small $\varepsilon$, we have $|v| \ge \frac{1}{2}\varepsilon^{-1/2}$ for all $v \in \Omega_\varepsilon$. 

On this domain $\Omega_\varepsilon$, \rev{the first Maxwellian component of $K_\varepsilon$ alone yields} $K_\varepsilon(v) \ge c_0 \varepsilon$ for some absolute constant $c_0 > 0$. Thus, the weighted norm is bounded below by:
\begin{align*}
    \|w K_\varepsilon\|_{L^2_v}^2 \ge \int_{\Omega_\varepsilon} w(v)^2 K_\varepsilon(v)^2 dv \ge c_0^2 \varepsilon^2 \int_{\Omega_\varepsilon} w(v)^2 dv.
\end{align*}
To bound the integral of $w(v)^2$ from below, we apply the Cauchy-Schwarz inequality over $\Omega_\varepsilon$:
\begin{align*}
    \left( \int_{\Omega_\varepsilon} |v|^2 dv \right)^2 = \left( \int_{\Omega_\varepsilon} \frac{|v|^2}{w(v)} w(v) dv \right)^2 \le \left( \int_{\Omega_\varepsilon} \frac{|v|^4}{w(v)^2} dv \right) \left( \int_{\Omega_\varepsilon} w(v)^2 dv \right).
\end{align*}
Since $|v| \ge \frac{1}{2}\varepsilon^{-1/2}$ on $\Omega_\varepsilon$, the integral on the left-hand side is bounded below by $\int_{\Omega_\varepsilon} \frac{1}{4}\varepsilon^{-1} dv = c_1 \varepsilon^{-1}$, where $c_1 = \frac{1}{4}|\Omega_\varepsilon| > 0$. Rearranging the inequality gives:
\begin{align*}
    \int_{\Omega_\varepsilon} w(v)^2 dv \ge \frac{c_1^2 \varepsilon^{-2}}{\int_{\Omega_\varepsilon} \frac{|v|^4}{w(v)^2} dv}.
\end{align*}
Combining these estimates, we obtain an explicit lower bound for the weighted loss:
\begin{align*}
    \|w K_\varepsilon\|_{L^2_v}^2 \ge c_0^2 \varepsilon^2 \left( \frac{c_1^2 \varepsilon^{-2}}{\int_{\Omega_\varepsilon} \frac{|v|^4}{w(v)^2} dv} \right) = \frac{c_0^2 c_1^2}{\int_{\Omega_\varepsilon} \frac{|v|^4}{w(v)^2} dv}.
\end{align*}
Recall that the fundamental integrability condition~\eqref{eq:weight_cond} guarantees $\int_{\mathbb{R}^3} \frac{|v|^4}{w(v)^2} dv \le \int_{\mathbb{R}^3} \frac{(1+|v|^2)^2}{w(v)^2} dv < \infty$. As $\varepsilon \to 0$, the domain $\Omega_\varepsilon$ shifts towards infinity ($|v| \to \infty$). Because the total integral over $\mathbb{R}^3$ is finite, the integral over the tail domain $\Omega_\varepsilon$ must vanish, i.e., $\lim_{\varepsilon \to 0} \int_{\Omega_\varepsilon} \frac{|v|^4}{w(v)^2} dv = 0$. 

Consequently, the denominator of our lower bound approaches zero from above, which implies $\|w K_\varepsilon\|_{L^2_v}^2 \to \infty$ as $\varepsilon \to 0$. This directly means that $\mathcal{L}_{w\text{-PINN}} \to \infty$ for both $\tilde{f}_\varepsilon^{(1)}$ and $\tilde{f}_\varepsilon^{(2)}$ presented in Section~\ref{sec:counterexamples}.
\end{Remark}

Furthermore, another notable advantage of this weighted norm is that it naturally yields rigorous control over the macroscopic physical quantities, in addition to the velocity distribution itself. As demonstrated in Section~\ref{sec:counterexamples}, the standard PINN approach fails because the $L^2$ smallness of the residual alone is insufficient to suppress the spurious production of macroscopic moments. The following corollary illustrates that minimizing the proposed weighted $L^2$ error effectively resolves this issue by guaranteeing the $L^1$ convergence of the macroscopic quantities.

\begin{corollary}[Convergence of Macroscopic Quantities] \label{cor:macro_convergence}
For any weight function $w(v)$ satisfying the integrability conditions~\eqref{eq:weight_cond}, the convergence of the weighted distance between $f$ and $\tilde{f}$ strictly guarantees the $L^1_x$ convergence of the macroscopic quantities. Specifically, there exists a constant $C>0$ depending on the weight $w:\mathbb{R}^3\rightarrow\mathbb{R}_+$ such that for all $t \in [0,\mathcal{T}]$,
\begin{align*}
    &\| \rho_f(t) - \rho_{\tilde{f}}(t) \|_{{L_{x}^1}} + \| (\rho_f u_f)(t) - (\rho_{\tilde{f}} u_{\tilde{f}})(t) \|_{{L_{x}^1}} + \| E_f(t) - E_{\tilde{f}}(t) \|_{{L_{x}^1}} \\
    &\leq C \| w(f-\tilde{f})(t) \|_{L_{x,v}^2},
\end{align*}
where $E_f = 3\rho_f T_f + \rho_f |u_f|^2$ denotes the macroscopic energy.
\end{corollary}
\begin{proof}
We denote the macroscopic moments generically by integrating against the collision invariants $\phi(v) \in \{1, v, |v|^2\}$. For any such $\phi(v)$, the difference in the macroscopic quantity can be bounded using the Cauchy-Schwarz inequality over the phase space $[0,1]^3 \times \mathbb{R}^3$:
\begin{align*}
    \int_{[0,1]^3} \left| \int_{\mathbb{R}^3} (f-\tilde{f}) \phi(v) dv \right| dx 
    &\leq \int_{[0,1]^3} \int_{\mathbb{R}^3} |w(v)(f-\tilde{f})| \frac{|\phi(v)|}{w(v)} dv dx \\
    &\leq \left( \int_{[0,1]^3} \int_{\mathbb{R}^3} |w(v)(f-\tilde{f})|^2 dv dx \right)^{1/2}\left( \int_{[0,1]^3} \int_{\mathbb{R}^3} \frac{|\phi(v)|^2}{w(v)^2} dv dx \right)^{1/2} \\
    &= \| w(f-\tilde{f})(t) \|_{L^2_{x,v}}\left( \big|[0,1]^3\big| \int_{\mathbb{R}^3} \frac{|\phi(v)|^2}{w(v)^2} dv \right)^{1/2}.
\end{align*}
Since $|\phi(v)|^2 \leq  (1+|v|^2)^2$ for all $v\in\mathbb{R}^3$, the integral $\int_{\mathbb{R}^3} \frac{|\phi(v)|^2}{w(v)^2} dv$ is strictly finite due to the first integrability condition in~\eqref{eq:weight_cond}. Letting this finite value be bounded by a constant $C_w$, setting $C = C_w^{1/2}$ yields the desired $L^1$ bound for the mass, momentum, and energy differences.
\end{proof}

\section{Numerical Experiments}
\label{sec:numerical_experiments}

In this section, we numerically validate the efficacy of the theory-guided weighted PINN loss, denoted as $\mathcal{L}_{w\text{-PINN}}$, derived in Section~\ref{sec:stability}. Specifically, we aim to verify the extent to which the proposed weighting scheme improves the prediction accuracy of the velocity distribution $f$ and, consequently, the macroscopic moments ($\rho, u, T$). The evaluation is based on the relative $L^1$ and $L^2$ errors, which were considered in Corollaries~\ref{cor:solution_convergence} and~\ref{cor:macro_convergence}.

To demonstrate the capabilities of our proposed method, the subsequent simulations cover a range of benchmark tests: \rev{spatially homogeneous problems, followed by problems with smooth initial data in one to three spatial dimensions and a one-dimensional Riemann problem,} with the microscopic velocity dimension consistently fixed at 3. To precisely assess the accuracy of the neural network predictions, we utilize the reference solutions which were computed employing a high-order conservative semi-Lagrangian scheme~\cite{cho2021conservative2, boscarino2020high}. \rev{The spatially homogeneous problems are evaluated directly against their analytic exact solutions.}

\subsection{Experimental Setup}
\label{subsec:experimental_setup}

To represent the approximate solution, we adopt the Separable Physics-Informed Neural Networks (SPINNs) architecture, which has demonstrated remarkable computational efficiency in high-dimensional spaces~\cite{cho2023separable}. \rev{Our construction of the approximate solution follows the work~\cite{oh2025separable} wherein the authors adapted the SPINN framework for solving the BGK model. Specifically, the approximate solution $f_\theta$ is defined by}
\begin{equation*}
    f_\theta = \mathcal{M}_{(\rho_\theta,\, u_\theta,\, T_\theta)} + e^{-\tau^2 |v - \mu|^2/2}\, f^{\mathrm{neq}}_\theta,
\end{equation*}
\rev{where $(\rho_\theta, u_\theta, T_\theta)$ are computed from $m = (m_1, \dots, m_5)$ given by}
\begin{equation*}
	m(t, x) = \sum_{r=1}^{R} G^{t}_r(t) \prod_{j=1}^{d} G^{x_j}_r(x_j),
\end{equation*}
\rev{where $G^{t} : \mathbb{R} \to \mathbb{R}^{5 \times R}$ and $G^{x_j} : \mathbb{R} \to \mathbb{R}^{R}$ are one-dimensional networks with sinusoidal activations~\cite{sitzmann2020implicit} assigned to $t$ and to each spatial coordinate, the subscript $r$ denotes the $r$-th component of their outputs (a vector in $\mathbb{R}^5$ for $G^{t}$ and a scalar for $G^{x_j}$), $d$ is the spatial dimension of the benchmark, and $R$ is the tensor rank. The macroscopic outputs are recovered as $\rho_\theta = e^{-m_1}$, $u_\theta = (m_2, m_3, m_4)$, and $T_\theta = e^{-m_5}$, so that $\rho_\theta$ and $T_\theta$ are positive by construction. The non-equilibrium correction $f^{\mathrm{neq}}_\theta$ is represented in the same separable form}
\begin{equation*}
	f^{\mathrm{neq}}_\theta(t, x, v)
	= \sum_{r=1}^{R} F^{t}_r(t) \prod_{j=1}^{d} F^{x_j}_r(x_j) \prod_{k=1}^{3} F^{v_k}_r(v_k),
\end{equation*}
\rev{where $F^{t}, F^{x_j}, F^{v_k} : \mathbb{R} \to \mathbb{R}^{R}$ are its own one-dimensional coordinate networks. The parameters $\tau$ and $\mu \in \mathbb{R}^3$ of the Gaussian factor are trained together with the networks rather than hand-tuned. More details on implementation and the choice of hyperparameters are provided in Appendix~\ref{app:repro}.} 

All numerical experiments were implemented using the JAX library and run on a single NVIDIA RTX A6000 GPU. \rev{The implementation is available from the authors upon request.} For each benchmark test, the full architectural and training details are provided in  Appendix~\ref{app:repro}.

For the training of the neural network, we employ the theory-guided weighted loss function $\mathcal{L}_{w\text{-PINN}}$ proposed in~\eqref{eq:w_pinn_loss_full}. \rev{In particular, we use the polynomial weight function $w(v) = 1 + \alpha |v|^\beta$ with $(\alpha, \beta) = (0.1, 4.0)$. This choice fulfills the condition $\beta > 7/2$ of Theorem~\ref{thm:main_stability}. A brief ablation study over $(\alpha, \beta)$ is provided in Appendix~\ref{app:ablation}.}

\subsection{Spatially Homogeneous Problems}
\label{subsec:homogeneous}

\rev{We first show that the inconsistency between the standard $L^2$ loss and the resulting error, observed in Section~\ref{sec:counterexamples}, indeed arises in actual training. The demonstration is based on the spatially homogeneous BGK equation $\partial_t f = \frac{1}{\mathrm{Kn}}(\mathcal{M}[f] - f)$. Here, the exact solution is explicitly given by
\begin{equation}\label{eq:homog_exact}
    f(t,v) = e^{-t/\mathrm{Kn}} f_0(v) + \left(1 - e^{-t/\mathrm{Kn}}\right) \mathcal{M}[f_0](v).
\end{equation}
All errors below are measured against this exact solution over the full time interval.}

\rev{We consider three families of initial data in which a non-negligible portion of the kinetic energy is carried by a low-density suprathermal component. Case A carries this energy in an isotropic hot tail, Case B in an anisotropic hot component aligned with $v_x$, and Case C in an isotropic suprathermal shell. The initial data are given by
\begin{align}
    \text{(Case A)} \quad
    & f_0(v) = (1-\delta)\, \mathcal{M}_{(1,\mathbf{0},T_c)}(v) + \delta\, \mathcal{M}_{(1,\mathbf{0},T_h)}(v), \qquad T_h = 3, \nonumber \\[2pt]
    \text{(Case B)} \quad
    & f_0(v) = (1-\delta)\, \mathcal{M}_{(1,\mathbf{0},T_c)}(v) + \frac{\delta}{(2\pi)^{3/2}\sqrt{H_x T_c^2}} \exp\!\left(-\frac{v_x^2}{2H_x} - \frac{v_y^2 + v_z^2}{2T_c}\right), \nonumber \\
    &\qquad\qquad H_x = 4, \nonumber \\[2pt]
    \text{(Case C)} \quad
    & f_0(v) = (1-\delta)\, \mathcal{M}_{(1,\mathbf{0},T_c)}(v) + \frac{\delta}{c_{v_0}}\, \exp\!\left(-\frac{(|v|-v_0)^2}{2\sigma^2}\right), \qquad v_0 = 4, \nonumber
\end{align}
where $T_c = 0.5$ and $\delta = 0.1$ in all three cases, $\sigma = 0.5$, and $c_{v_0} > 0$ normalizes the shell component to unit mass. We set $\mathrm{Kn} = 0.1$. Additional configurations of the three families ($T_h \in \{4,5\}$, $H_x \in \{3,5\}$, and $v_0 \in \{2,3\}$) are reported in Appendix~\ref{app:homog-extra}.}

\rev{Table~\ref{tab:homog_results} reports the final training loss and the relative $L^2$ errors of the distribution function $f$ and the temperature $T$ attained by the standard and the proposed weighted losses. In accuracy, the weighted loss outperforms the standard loss, not only for the distribution function $f$ but also for the temperature $T$. This gap directly exhibits the energy-moment corruption predicted by the counterexamples of Section~\ref{sec:counterexamples}. Figure~\ref{fig:homog_lossacc} shows the relation between the training loss and the error. For the standard $L^2$ loss, the error decreases only slowly with the training loss, and in Case C it even flattens while the loss keeps decreasing. In contrast, for the weighted loss, the error exhibits a clear decreasing trend with the training loss in all cases, in agreement with the stability estimate of Theorem~\ref{thm:main_stability}.}

\begin{table}[h]
\centering
\caption{\rev{Final training loss and relative $L^2$ errors of the distribution function $f$ and the temperature $T$ (the energy moment) for the spatially homogeneous problems. Values are the mean $\pm$ standard deviation over five runs.}}
\label{tab:homog_results}
\resizebox{\textwidth}{!}{%
\setlength{\tabcolsep}{6pt}%
\begin{tabular}{lcccccc}
\toprule
& \multicolumn{3}{c}{\textbf{Standard $L^2$ Loss}} & \multicolumn{3}{c}{\textbf{Proposed Loss}} \\
\cmidrule(lr){2-4} \cmidrule(lr){5-7}
\textbf{Initial datum} & \textbf{Training loss} & $\boldsymbol{f}$ \textbf{(Relative $L^2$)} & $\boldsymbol{T}$ \textbf{(Relative $L^2$)} & \textbf{Training loss} & $\boldsymbol{f}$ \textbf{(Relative $L^2$)} & $\boldsymbol{T}$ \textbf{(Relative $L^2$)} \\
\midrule
Case A ($T_h=3$) & $(2.48 \pm 0.21) \times 10^{-10}$ & $(1.13 \pm 0.04) \times 10^{-2}$ & $(3.28 \pm 0.12) \times 10^{-2}$ & $(3.09 \pm 0.17) \times 10^{-9}$ & $\mathbf{(2.31 \pm 0.46) \times 10^{-4}}$ & $\mathbf{(6.26 \pm 0.50) \times 10^{-4}}$ \\  
Case B ($H_x=4$) & $(5.13 \pm 0.22) \times 10^{-11}$ & $(1.85 \pm 0.04) \times 10^{-3}$ & $(4.97 \pm 0.16) \times 10^{-3}$ & $(1.35 \pm 0.14) \times 10^{-10}$ & $\mathbf{(1.04 \pm 0.06) \times 10^{-4}}$ & $\mathbf{(5.10 \pm 2.00) \times 10^{-5}}$ \\  
Case C ($v_0=4$) & $(3.84 \pm 3.13) \times 10^{-7}$ & $(1.47 \pm 1.17) \times 10^{-1}$ & $(3.11 \pm 2.46) \times 10^{-1}$ & $(7.18 \pm 11.05) \times 10^{-9}$ & $\mathbf{(1.76 \pm 0.90) \times 10^{-4}}$ & $\mathbf{(1.95 \pm 1.24) \times 10^{-4}}$ \\  
\bottomrule
\end{tabular}
}
\end{table}

\begin{figure}[hbt!]
    \centering
    \includegraphics[width=\textwidth]{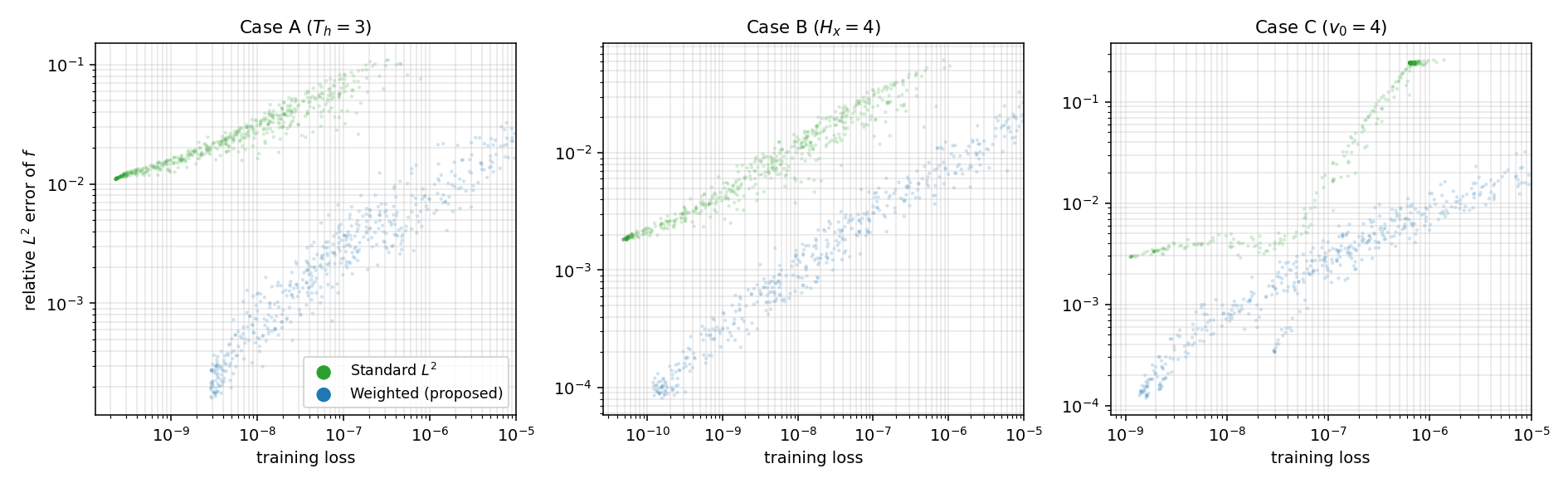}
    \caption{\rev{Training loss versus relative $L^2$ error of $f$ for the three spatially homogeneous cases.}}
    \label{fig:homog_lossacc}
\end{figure}

\subsection{1D Smooth Problem}
\label{subsec:1d_smooth}

\rev{We now turn to spatially inhomogeneous benchmarks.} The one-dimensional smooth problem is given on the spatial domain $x \in (-0.5, 0.5)$ with periodic boundary conditions. Here, the initial condition is given by a Maxwellian distribution characterized by the following macroscopic moments:
\begin{equation}
    \rho_0(x) = 1 + 0.5 \sin(2\pi x), \quad u_0(x) = 0, \quad T_0(x) = 1 + 0.5 \sin(2\pi x + 0.2).
\end{equation}
\rev{In addition to the standard $L^2$ loss, we also compare with the relative loss proposed in~\cite{oh2025separable}, which applies a weight of approximately $1/(|f_\theta|+\epsilon)$ to the residual. Following the original work, we set $\epsilon = 10^{-3}$.}

\rev{Table~\ref{tab:smooth_1d_errors} summarizes the relative errors of the distribution function and the macroscopic moments, evaluated at the terminal time $t = 0.1$, for $\mathrm{Kn} = 1.0, 0.1$, and $0.01$. The proposed loss achieves the lowest errors in every metric and at every Knudsen number. In Figure~\ref{fig:smooth_comparison_kn001}, we compare the numerical solutions obtained at $\mathrm{Kn}=0.01$.}
\begin{table}[h]
\centering
\caption{Relative errors of the distribution function $f$ and macroscopic moments ($\rho, u_x, T$) for the 1D smooth problem across different Knudsen numbers ($\mathrm{Kn} \in \{1.0, 0.1, 0.01\}$). \rev{Values are the mean $\pm$ standard deviation over five runs.}}
\label{tab:smooth_1d_errors}
\resizebox{\textwidth}{!}{%
\setlength{\tabcolsep}{6pt}%
\begin{tabular}{clccccc}
\toprule
$\mathrm{Kn}$ & \textbf{Loss function} & $\boldsymbol{f}$ \textbf{(Relative $L^2$)} & $\boldsymbol{f}$ \textbf{(Relative $L^1$)} & $\boldsymbol{\rho}$ \textbf{(Relative $L^1$)} & $\boldsymbol{u_x}$ \textbf{(Relative $L^1$)} & $\boldsymbol{T}$ \textbf{(Relative $L^1$)} \\
\midrule
\multirow{3}{*}{$1.0$}
& Standard $L^2$ Loss & $(1.39 \pm 0.06) \times 10^{-3}$ & $(2.82 \pm 0.11) \times 10^{-3}$ & $(7.98 \pm 0.72) \times 10^{-4}$ & $(1.52 \pm 0.14) \times 10^{-3}$ & $(3.34 \pm 0.25) \times 10^{-3}$ \\
& Relative Loss~\cite{oh2025separable} & $(2.66 \pm 0.20) \times 10^{-3}$ & $(2.40 \pm 0.11) \times 10^{-3}$ & $(4.60 \pm 0.37) \times 10^{-4}$ & $(1.37 \pm 0.07) \times 10^{-3}$ & $(2.48 \pm 0.20) \times 10^{-3}$ \\
& \textbf{Proposed Loss} & $\mathbf{(6.05 \pm 0.34) \times 10^{-4}}$ & $\mathbf{(7.08 \pm 0.21) \times 10^{-4}}$ & $\mathbf{(4.48 \pm 0.26) \times 10^{-5}}$ & $\mathbf{(8.57 \pm 0.61) \times 10^{-5}}$ & $\mathbf{(1.82 \pm 0.22) \times 10^{-4}}$ \\
\midrule
\multirow{3}{*}{$0.1$}
& Standard $L^2$ Loss & $(1.91 \pm 0.09) \times 10^{-3}$ & $(2.58 \pm 0.11) \times 10^{-3}$ & $(6.57 \pm 0.52) \times 10^{-4}$ & $(1.77 \pm 0.17) \times 10^{-3}$ & $(2.14 \pm 0.33) \times 10^{-3}$ \\
& Relative Loss~\cite{oh2025separable} & $(1.70 \pm 0.13) \times 10^{-3}$ & $(1.34 \pm 0.06) \times 10^{-3}$ & $(3.49 \pm 0.18) \times 10^{-4}$ & $(8.46 \pm 0.59) \times 10^{-4}$ & $(1.48 \pm 0.10) \times 10^{-3}$ \\
& \textbf{Proposed Loss} & $\mathbf{(3.07 \pm 0.16) \times 10^{-4}}$ & $\mathbf{(3.34 \pm 0.10) \times 10^{-4}}$ & $\mathbf{(4.01 \pm 1.10) \times 10^{-5}}$ & $\mathbf{(4.30 \pm 0.69) \times 10^{-5}}$ & $\mathbf{(7.34 \pm 0.84) \times 10^{-5}}$ \\
\midrule
\multirow{3}{*}{$0.01$}
& Standard $L^2$ Loss & $(2.77 \pm 0.42) \times 10^{-2}$ & $(1.45 \pm 0.22) \times 10^{-2}$ & $(1.21 \pm 0.17) \times 10^{-3}$ & $(6.21 \pm 0.81) \times 10^{-3}$ & $(1.29 \pm 0.20) \times 10^{-2}$ \\
& Relative Loss~\cite{oh2025separable} & $(1.41 \pm 0.07) \times 10^{-3}$ & $(6.39 \pm 0.49) \times 10^{-4}$ & $(3.38 \pm 0.18) \times 10^{-4}$ & $(2.86 \pm 0.38) \times 10^{-4}$ & $(4.17 \pm 0.48) \times 10^{-4}$ \\
& \textbf{Proposed Loss} & $\mathbf{(1.50 \pm 0.12) \times 10^{-4}}$ & $\mathbf{(1.14 \pm 0.10) \times 10^{-4}}$ & $\mathbf{(4.92 \pm 1.10) \times 10^{-5}}$ & $\mathbf{(5.39 \pm 1.30) \times 10^{-5}}$ & $\mathbf{(5.71 \pm 0.76) \times 10^{-5}}$ \\
\bottomrule
\end{tabular}
}
\end{table}

\begin{figure}[h]
    \centering
    \includegraphics[width=0.9\textwidth]{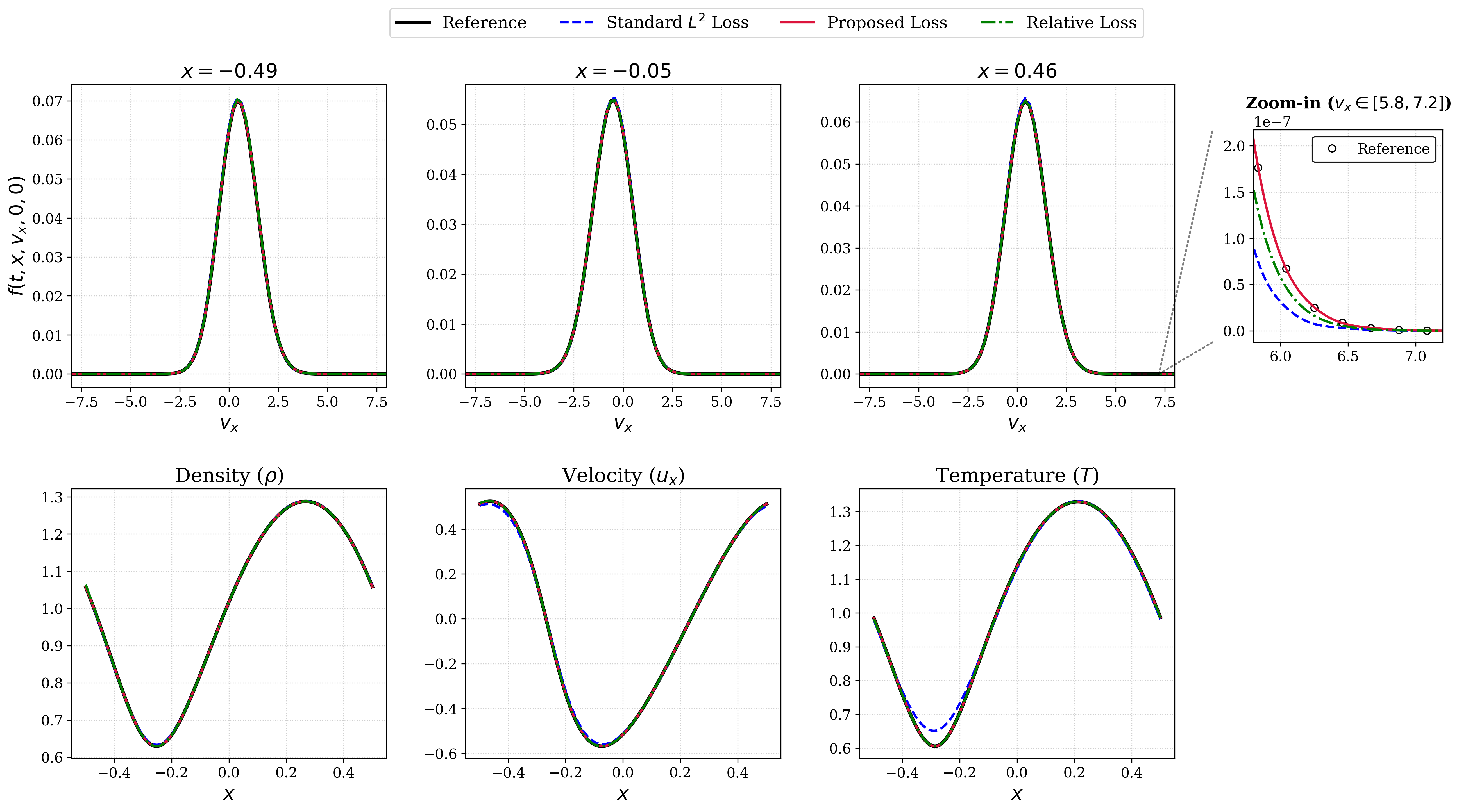}
    \caption{Distribution function $f$ (top row) and macroscopic moments ($\rho, u_x, T$) (bottom row) at $t=0.1$ for the 1D Smooth problem at $\mathrm{Kn}=0.01$.}
    \label{fig:smooth_comparison_kn001}
\end{figure}
\rev{Table~\ref{tab:smooth1d_cost} compares the training cost of the three losses on this problem. The proposed weight adds no computational overhead in either time or memory.}
\begin{table}[h]
\centering
\caption{\rev{Per-iteration training time and peak GPU memory for the 1D smooth problem at $\mathrm{Kn}=0.01$.}}
\label{tab:smooth1d_cost}
\footnotesize
\rev{%
\begin{tabular}{lcc}
\toprule
\textbf{Loss function} & \textbf{Time per iteration (ms)} & \textbf{Peak GPU memory (MB)} \\
\midrule
Standard $L^2$ Loss & $1.019$ & $163$ \\
Relative Loss~\cite{oh2025separable} & $1.015$ & $163$ \\
Proposed Loss & $1.016$ & $163$ \\
\bottomrule
\end{tabular}}
\end{table}
\rev{Figure~\ref{fig:smooth1d_epoch_error} shows the evolution of the errors of $f$ and $T$ during training at $\mathrm{Kn} = 0.01$. While the three losses require nearly identical
	per-iteration costs as reported in Table~\ref{tab:smooth1d_cost}, the proposed loss
	decreases the errors far more rapidly than the two baselines.}
\begin{figure}[h]
    \centering
    \includegraphics[width=0.9\textwidth]{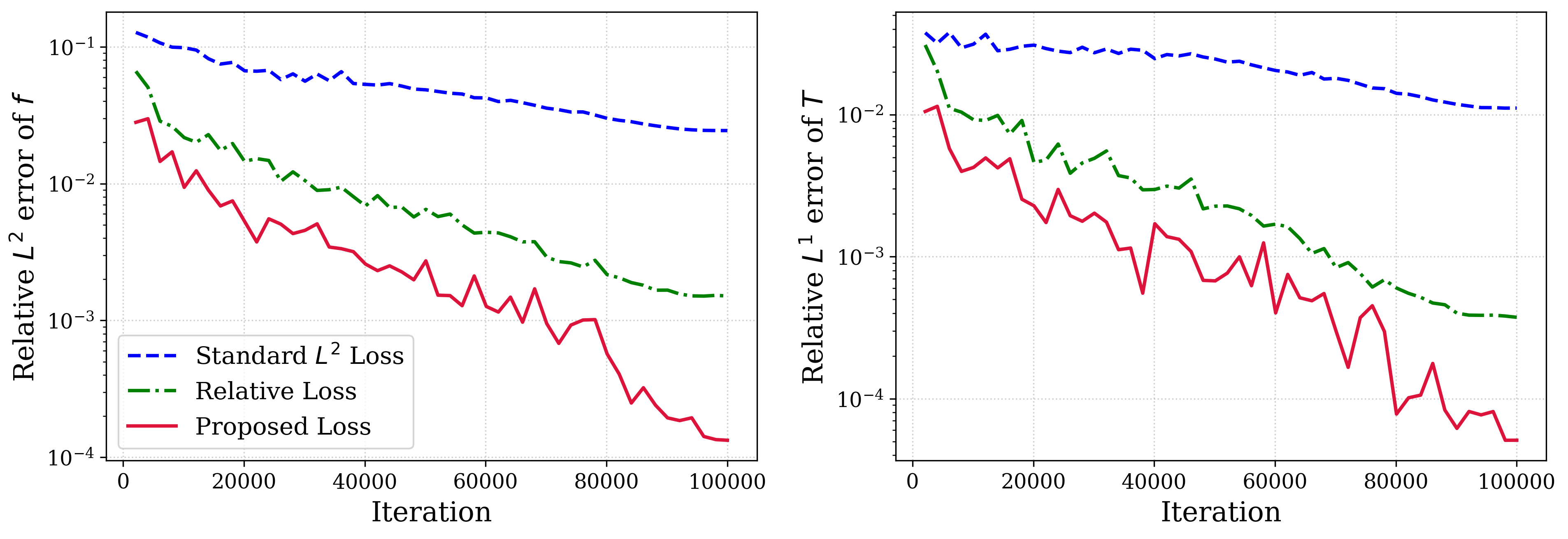}
    \caption{\rev{Relative errors of $f$ (left) and $T$ (right) over the course of training for the 1D smooth problem at $\mathrm{Kn}=0.01$.}}
    \label{fig:smooth1d_epoch_error}
\end{figure}

\subsection{2D Smooth Problem}
\label{subsec:2d_smooth}

To further validate the performance of the proposed weighted loss across different spatial dimensions, we simulate a two-dimensional smooth problem. This benchmark extends the phase space to five dimensions ($x, y, v_x,$ $v_y, v_z$). The problem is defined on the spatial domain $(x, y) \in (-0.5, 0.5)^2$ with periodic boundary conditions applied across both spatial directions. The initial condition is given by a Maxwellian distribution characterized by the following macroscopic moments:
\begin{align}
    &\rho_0(x, y) = 1 + 0.5 \sin(2\pi x)\sin(2\pi y), \quad u_0(x, y) = \mathbf{0}, \nonumber\\
    &\rev{T_0(x, y) = 1 + 0.5 \sin(2\pi (x + y) + 0.2).}
\end{align}

Table~\ref{tab:smooth_2d_errors} summarizes \rev{the relative $L^1$ errors of} the macroscopic moments\rev{, evaluated at the terminal time $t = 0.1$,} across the 2D spatial domain. \rev{The proposed loss attains the smallest error in every metric at all three Knudsen numbers.}
\begin{table}[h]
\centering
\caption{Relative errors of the macroscopic moments ($\rho, u_x, u_y, T$) for the 2D smooth problem across different Knudsen numbers ($\mathrm{Kn} \in \{1.0, 0.1, 0.01\}$). \rev{Values are the mean $\pm$ standard deviation over five runs.}}
\label{tab:smooth_2d_errors}
\resizebox{\textwidth}{!}{%
\setlength{\tabcolsep}{14.9pt}%
\begin{tabular}{clcccc}
\toprule
$\mathrm{Kn}$ & \textbf{Loss function} & $\boldsymbol{\rho}$ \textbf{(Relative $L^1$)} & $\boldsymbol{u_x}$ \textbf{(Relative $L^1$)} & $\boldsymbol{u_y}$ \textbf{(Relative $L^1$)} & $\boldsymbol{T}$ \textbf{(Relative $L^1$)} \\
\midrule
\multirow{3}{*}{$1.0$}
& Standard $L^2$ Loss & $(1.93 \pm 0.11) \times 10^{-3}$ & $(3.04 \pm 0.15) \times 10^{-2}$ & $(2.88 \pm 0.25) \times 10^{-2}$ & $(6.91 \pm 0.44) \times 10^{-3}$ \\
& Relative Loss~\cite{oh2025separable} & $(1.58 \pm 0.17) \times 10^{-3}$ & $(1.63 \pm 0.07) \times 10^{-2}$ & $(1.57 \pm 0.04) \times 10^{-2}$ & $(1.50 \pm 0.15) \times 10^{-3}$ \\
& \textbf{Proposed Loss} & $\mathbf{(4.23 \pm 0.61) \times 10^{-4}}$ & $\mathbf{(7.54 \pm 0.50) \times 10^{-3}}$ & $\mathbf{(7.28 \pm 0.35) \times 10^{-3}}$ & $\mathbf{(7.25 \pm 0.66) \times 10^{-4}}$ \\
\midrule
\multirow{3}{*}{$0.1$}
& Standard $L^2$ Loss & $(1.94 \pm 0.15) \times 10^{-3}$ & $(2.77 \pm 0.35) \times 10^{-2}$ & $(2.94 \pm 0.31) \times 10^{-2}$ & $(3.75 \pm 0.56) \times 10^{-3}$ \\
& Relative Loss~\cite{oh2025separable} & $(1.81 \pm 0.04) \times 10^{-3}$ & $(8.27 \pm 0.52) \times 10^{-3}$ & $(8.55 \pm 0.60) \times 10^{-3}$ & $(1.55 \pm 0.10) \times 10^{-3}$ \\
& \textbf{Proposed Loss} & $\mathbf{(4.05 \pm 0.68) \times 10^{-4}}$ & $\mathbf{(3.00 \pm 0.26) \times 10^{-3}}$ & $\mathbf{(2.84 \pm 0.41) \times 10^{-3}}$ & $\mathbf{(5.04 \pm 0.60) \times 10^{-4}}$ \\
\midrule
\multirow{3}{*}{$0.01$}
& Standard $L^2$ Loss & $(8.15 \pm 0.44) \times 10^{-3}$ & $(1.14 \pm 0.08) \times 10^{-1}$ & $(1.12 \pm 0.08) \times 10^{-1}$ & $(2.57 \pm 0.33) \times 10^{-2}$ \\
& Relative Loss~\cite{oh2025separable} & $(1.93 \pm 0.10) \times 10^{-3}$ & $(7.28 \pm 0.19) \times 10^{-3}$ & $(7.51 \pm 0.43) \times 10^{-3}$ & $(2.35 \pm 0.07) \times 10^{-3}$ \\
& \textbf{Proposed Loss} & $\mathbf{(8.18 \pm 0.32) \times 10^{-4}}$ & $\mathbf{(4.74 \pm 0.10) \times 10^{-3}}$ & $\mathbf{(4.77 \pm 0.14) \times 10^{-3}}$ & $\mathbf{(9.99 \pm 0.16) \times 10^{-4}}$ \\
\bottomrule
\end{tabular}
}
\end{table}
\subsection{3D Smooth Problem}
\label{subsec:3d_smooth}

To verify the effectiveness of the proposed weighted loss in a full high-dimensional setting, we simulate a three-dimensional smooth problem. This problem is challenging as the phase space extends to six dimensions ($x, y, z,$ $v_x, v_y, v_z$). This benchmark is defined on the spatial domain $(x, y, z) \in (-0.5, 0.5)^3$ with periodic boundary conditions applied across all spatial dimensions. The initial condition is given by a Maxwellian distribution characterized by the following macroscopic moments:
\begin{align}
    &\rho_0(x, y, z) = 1 + 0.5 \sin(2\pi x)\sin(2\pi y)\sin(2\pi z), \quad u_0(x, y, z) = \mathbf{0}, \nonumber\\
    &\rev{T_0(x, y, z) = 1 + 0.5 \sin(2\pi (x + y + z) + 0.2).}
\end{align}

\rev{The errors of the macroscopic moments are evaluated at the terminal time $t = 0.1$, and the results at $\mathrm{Kn} = 1.0$ and $0.01$ are compared against the baseline models.}

\begin{table}[htbp]
\centering
\caption{Relative errors of the \rev{macroscopic moments ($\rho, u_x, u_y, u_z, T$)} for the 3D smooth problem across different Knudsen numbers \rev{($\mathrm{Kn} \in \{1.0, 0.01\}$)}. \rev{Values are the mean $\pm$ standard deviation over three runs.}}
\label{tab:smooth_3d_errors}
\resizebox{\textwidth}{!}{%
\setlength{\tabcolsep}{6.0pt}%
\begin{tabular}{clccccc}
\toprule
$\mathrm{Kn}$ & \textbf{Loss function} & $\boldsymbol{\rho}$ \textbf{(Relative $L^1$)} & $\boldsymbol{u_x}$ \textbf{(Relative $L^1$)} & $\boldsymbol{u_y}$ \textbf{(Relative $L^1$)} & $\boldsymbol{u_z}$ \textbf{(Relative $L^1$)} & $\boldsymbol{T}$ \textbf{(Relative $L^1$)} \\
\midrule
\multirow{3}{*}{$1.0$}
& Standard $L^2$ Loss & $(2.40 \pm 0.20) \times 10^{-3}$ & $(9.51 \pm 0.65) \times 10^{-2}$ & $(9.63 \pm 0.38) \times 10^{-2}$ & $(9.50 \pm 0.57) \times 10^{-2}$ & $(6.86 \pm 0.94) \times 10^{-3}$ \\
& Relative Loss~\cite{oh2025separable} & $(8.26 \pm 0.71) \times 10^{-3}$ & $(1.05 \pm 0.19) \times 10^{-1}$ & $(1.02 \pm 0.19) \times 10^{-1}$ & $(9.64 \pm 1.12) \times 10^{-2}$ & $(5.76 \pm 0.61) \times 10^{-3}$ \\
& \textbf{Proposed Loss} & $\mathbf{(2.18 \pm 0.48) \times 10^{-3}}$ & $\mathbf{(4.66 \pm 0.23) \times 10^{-2}}$ & $\mathbf{(4.29 \pm 0.35) \times 10^{-2}}$ & $\mathbf{(4.57 \pm 0.27) \times 10^{-2}}$ & $\mathbf{(2.43 \pm 0.75) \times 10^{-3}}$ \\
\midrule
\multirow{3}{*}{$0.01$}
& Standard $L^2$ Loss & $(3.83 \pm 0.30) \times 10^{-2}$ & $(3.57 \pm 0.28) \times 10^{-1}$ & $(3.57 \pm 0.20) \times 10^{-1}$ & $(3.57 \pm 0.23) \times 10^{-1}$ & $(6.34 \pm 0.60) \times 10^{-2}$ \\
& Relative Loss~\cite{oh2025separable} & $(8.08 \pm 0.14) \times 10^{-3}$ & $(3.95 \pm 0.28) \times 10^{-2}$ & $(3.95 \pm 0.34) \times 10^{-2}$ & $(3.91 \pm 0.26) \times 10^{-2}$ & $(7.60 \pm 0.33) \times 10^{-3}$ \\
& \textbf{Proposed Loss} & $\mathbf{(4.23 \pm 0.13) \times 10^{-3}}$ & $\mathbf{(2.88 \pm 0.02) \times 10^{-2}}$ & $\mathbf{(2.88 \pm 0.10) \times 10^{-2}}$ & $\mathbf{(2.95 \pm 0.05) \times 10^{-2}}$ & $\mathbf{(4.10 \pm 0.19) \times 10^{-3}}$ \\
\bottomrule
\end{tabular}
}
\end{table}

Table~\ref{tab:smooth_3d_errors} summarizes the prediction errors for the macroscopic moments across the 3D spatial domain. \rev{The proposed loss attains the lowest mean error in every metric at both Knudsen numbers. That is, these results confirm that the advantage of the proposed weighted loss carries over to the full six-dimensional phase space.}

\subsection{1D Riemann Problem}
\label{subsec:1d_riemann}

\rev{Lastly, moving beyond smooth initial data,} we simulate a one-dimensional Riemann problem. This benchmark, typically representing a shock tube scenario, is characterized by piecewise constant initial macroscopic states that generate complex wave structures, including shock waves, contact discontinuities, and rarefaction waves.

Specifically, it is given on the spatial domain $x \in (-0.5, 0.5)$ with homogeneous Neumann boundary conditions applied at $x = \pm 0.5$. To facilitate the neural network training, the initial discontinuity is slightly smoothed using a hyperbolic tangent function. The initial condition is given by a Maxwellian distribution constructed from the following macroscopic states:
\begin{equation}
    \rho_0(x) = 1 - 0.875 h(x), \quad u_0(x) = 0, \quad \rev{T_0(x) = 1 - 0.5 h(x),}
\end{equation}
where $h(x) = \frac{1}{2}(1 + \tanh(100x))$ serves as a smoothed Heaviside step function. This effectively defines the left state ($x < 0$) as $(\rho_L, u_L, T_L) \approx \rev{(1.0, 0, 1.0)}$ and the right state ($x > 0$) as $(\rho_R, u_R, T_R) \approx \rev{(0.125, 0, 0.5)}$.

Table~\ref{tab:riemann_errors} summarizes the relative $L^2$ and $L^1$ errors\rev{, evaluated at the terminal time $t = 0.1$,} for the distribution function $f$ and the macroscopic moments across three distinct Knudsen numbers: $\mathrm{Kn} = 1.0, 0.1,$ and $0.01$.
\begin{table}[h]
\centering
\caption{Relative errors of the distribution function $f$ and macroscopic moments ($\rho, u_x, T$) for the 1D Riemann problem across different Knudsen numbers ($\mathrm{Kn} \in \{1.0, 0.1, 0.01\}$). \rev{Values are the mean $\pm$ standard deviation over five runs.}}
\label{tab:riemann_errors}
\resizebox{\textwidth}{!}{%
\setlength{\tabcolsep}{6pt}%
\begin{tabular}{clccccc}
\toprule
$\mathrm{Kn}$ & \textbf{Loss function} & $\boldsymbol{f}$ \textbf{(Relative $L^2$)} & $\boldsymbol{f}$ \textbf{(Relative $L^1$)} & $\boldsymbol{\rho}$ \textbf{(Relative $L^1$)} & $\boldsymbol{u_x}$ \textbf{(Relative $L^1$)} & $\boldsymbol{T}$ \textbf{(Relative $L^1$)} \\
\midrule
\multirow{3}{*}{$1.0$}
& Standard $L^2$ Loss & $(1.32 \pm 0.04) \times 10^{-2}$ & $(9.11 \pm 0.29) \times 10^{-3}$ & $(1.16 \pm 0.08) \times 10^{-3}$ & $(2.22 \pm 0.06) \times 10^{-2}$ & $(6.67 \pm 0.54) \times 10^{-3}$ \\
& Relative Loss~\cite{oh2025separable} & $(2.54 \pm 0.14) \times 10^{-2}$ & $(1.27 \pm 0.05) \times 10^{-2}$ & $(1.50 \pm 0.16) \times 10^{-3}$ & $(1.63 \pm 0.13) \times 10^{-2}$ & $(3.22 \pm 0.38) \times 10^{-3}$ \\
& \textbf{Proposed Loss} & $\mathbf{(1.29 \pm 0.07) \times 10^{-2}}$ & $\mathbf{(8.83 \pm 0.58) \times 10^{-3}}$ & $\mathbf{(7.29 \pm 0.90) \times 10^{-4}}$ & $\mathbf{(6.16 \pm 1.55) \times 10^{-3}}$ & $\mathbf{(1.40 \pm 0.08) \times 10^{-3}}$ \\
\midrule
\multirow{3}{*}{$0.1$}
& Standard $L^2$ Loss & $(5.69 \pm 0.08) \times 10^{-3}$ & $(4.29 \pm 0.04) \times 10^{-3}$ & $(5.99 \pm 0.54) \times 10^{-4}$ & $(1.16 \pm 0.06) \times 10^{-2}$ & $(3.24 \pm 0.39) \times 10^{-3}$ \\
& Relative Loss~\cite{oh2025separable} & $(1.02 \pm 0.06) \times 10^{-2}$ & $(5.56 \pm 0.34) \times 10^{-3}$ & $(9.95 \pm 2.13) \times 10^{-4}$ & $(9.76 \pm 1.19) \times 10^{-3}$ & $(1.93 \pm 0.21) \times 10^{-3}$ \\
& \textbf{Proposed Loss} & $\mathbf{(5.51 \pm 0.18) \times 10^{-3}}$ & $\mathbf{(3.98 \pm 0.19) \times 10^{-3}}$ & $\mathbf{(4.35 \pm 0.96) \times 10^{-4}}$ & $\mathbf{(3.56 \pm 0.74) \times 10^{-3}}$ & $\mathbf{(6.86 \pm 0.72) \times 10^{-4}}$ \\
\midrule
\multirow{3}{*}{$0.01$}
& Standard $L^2$ Loss & $(4.12 \pm 0.17) \times 10^{-2}$ & $(2.71 \pm 0.13) \times 10^{-2}$ & $(8.65 \pm 0.37) \times 10^{-3}$ & $(1.15 \pm 0.04) \times 10^{-1}$ & $(6.85 \pm 0.32) \times 10^{-2}$ \\
& Relative Loss~\cite{oh2025separable} & $(6.74 \pm 0.36) \times 10^{-3}$ & $(3.69 \pm 0.18) \times 10^{-3}$ & $(1.83 \pm 0.07) \times 10^{-3}$ & $(2.09 \pm 0.10) \times 10^{-2}$ & $(6.74 \pm 0.49) \times 10^{-3}$ \\
& \textbf{Proposed Loss} & $\mathbf{(4.56 \pm 0.32) \times 10^{-3}}$ & $\mathbf{(2.73 \pm 0.24) \times 10^{-3}}$ & $\mathbf{(1.16 \pm 0.14) \times 10^{-3}}$ & $\mathbf{(2.02 \pm 0.26) \times 10^{-2}}$ & $\mathbf{(5.12 \pm 0.41) \times 10^{-3}}$ \\
\bottomrule
\end{tabular}
}
\end{table}

\rev{At $\mathrm{Kn} = 1.0$ and $0.1$, the proposed loss attains an accuracy comparable to the standard $L^2$ loss for the distribution function, while predicting the macroscopic moments, in particular the temperature, far more accurately. The relative loss is even less accurate than the standard loss for the distribution function in these regimes. At $\mathrm{Kn} = 0.01$, both the relative and the proposed losses clearly outperform the standard loss, and between the two the proposed loss attains the lower mean error in every metric. The proposed loss is thus the only method that performs uniformly well over all the tested cases, which indicates that its benefit is not restricted to smooth solutions.} Figure~\ref{fig:riemann_comparison_kn001} displays the distribution function at three different locations and macroscopic variables at $\mathrm{Kn}=0.01$.
\begin{figure}[h]
    \centering
    \includegraphics[width=\textwidth]{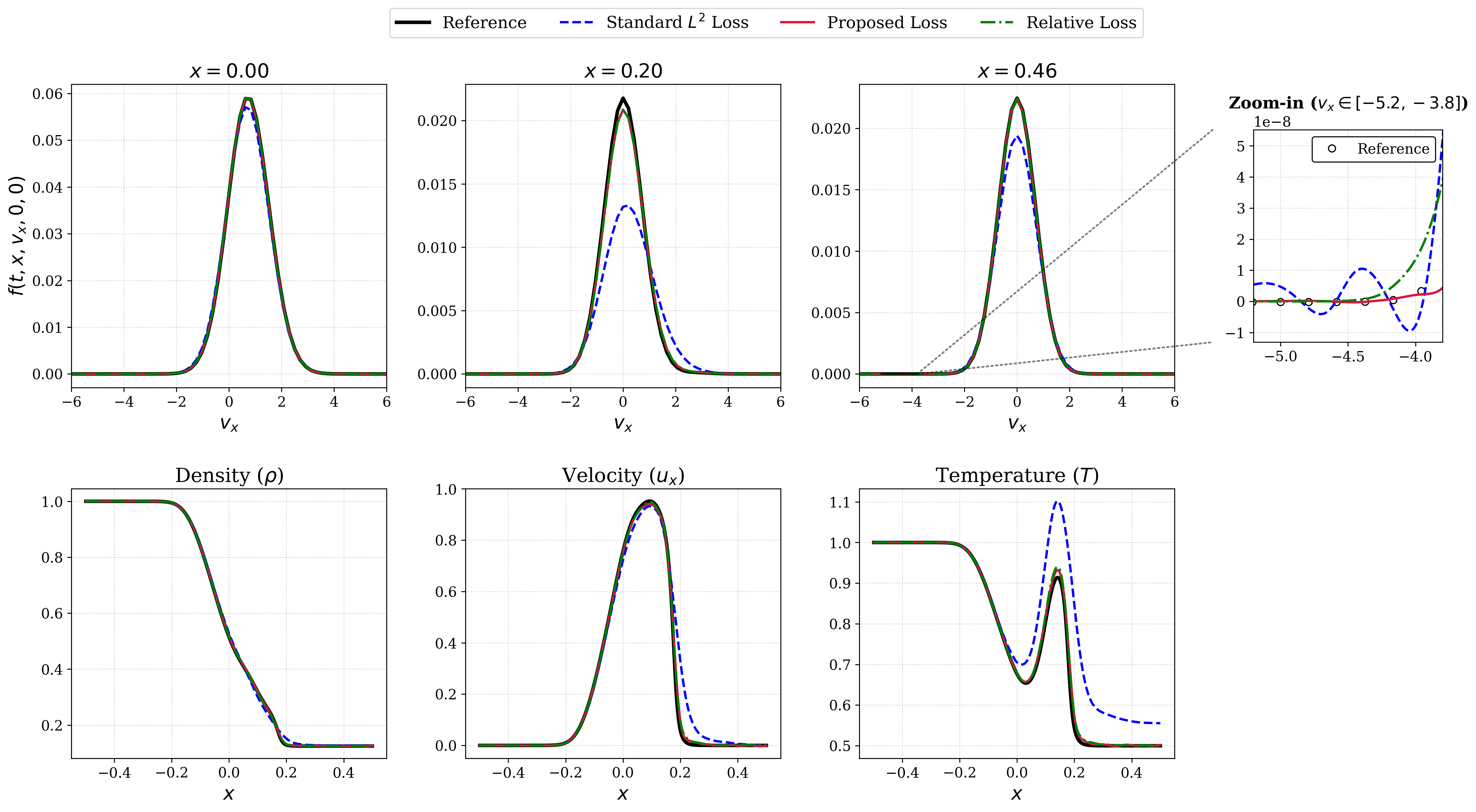}
    \caption{Distribution function $f$ (top row) and macroscopic moments ($\rho, u_x, T$) (bottom row) at $t=0.1$ for the 1D Riemann problem at $\mathrm{Kn}=0.01$.}
    \label{fig:riemann_comparison_kn001}
\end{figure}

\subsection{Discussion on the results}
\label{subsec:discussion_consistency}

The numerical results show that our proposed weight function, $w(v) = 1 + \alpha |v|^\beta$, consistently provides stable and accurate predictions across all tested cases. \rev{However, in the Navier-Stokes regime ($\mathrm{Kn}=0.01$), the relative loss remains competitive (see Tables~\ref{tab:smooth_2d_errors}--\ref{tab:riemann_errors}). This can be explained by the structural similarity between the two weight functions.}

Figure~\ref{fig:weight_shape_comparison} compares the proposed polynomial weight ($\alpha=0.1, \beta=4.0$) with the weight of the relative loss method ($1/(|f_\theta|+\epsilon)$) for the 1D smooth problem. While the exact profiles of the two weights are not identical, they exhibit a shared qualitative mechanism: both assign a relatively small penalty near the center velocity region ($v \approx 0$) and impose increasingly larger weights for velocities in the tail regions. This common property provides a plausible explanation for why the relative loss can achieve good performance in certain cases. 
\begin{figure}[t!]
    \centering
    \includegraphics[width=\textwidth]{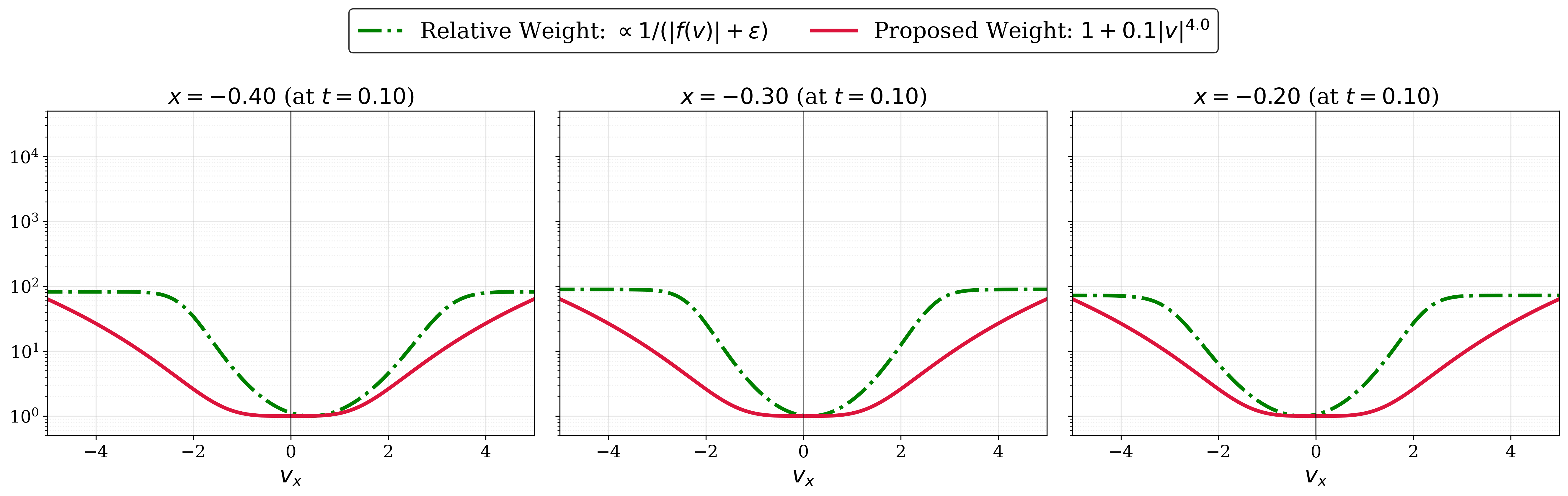}
    \caption{Comparison of weight shapes between the relative loss and the proposed loss for the 1D smooth problem at $\mathrm{Kn}=0.01$.}
    \label{fig:weight_shape_comparison}
\end{figure}
Figure~\ref{fig:weight_shape_comparison_riemann} illustrates the weight profiles evaluated for the 1D Riemann problem. Because the relative weight ($\propto 1/(|f_\theta|+\epsilon)$) depends directly on the distribution function itself, it exhibits a more irregular and fluctuating shape for solutions containing complex features such as discontinuities. Although it generally assigns larger weights away from the center, the profile exhibits steep slopes and bends, differing from the proposed polynomial weight. Indeed, when using this relative loss for the 1D Riemann problem, its accuracy dropped in certain cases, sometimes performing even worse than the standard $L^2$ loss (see Relative $L^2$ and $L^1$ for $f$ in Table~\ref{tab:riemann_errors}). \rev{A similar degradation is observed even in the absence of discontinuities. In the three-dimensional smooth problem, the relative loss predicts the density at $\mathrm{Kn}=1.0$ less accurately than the standard loss (Table~\ref{tab:smooth_3d_errors}). Since the relative weight depends on the network output $f_\theta$ itself, an inaccurate $f_\theta$---near discontinuities or in high-dimensional problems---yields an unreliable weight that further misguides the training. In contrast, the proposed polynomial weight is fixed in advance and thus yields robust results.}

\begin{figure}[t!]
    \centering
    \includegraphics[width=\textwidth]{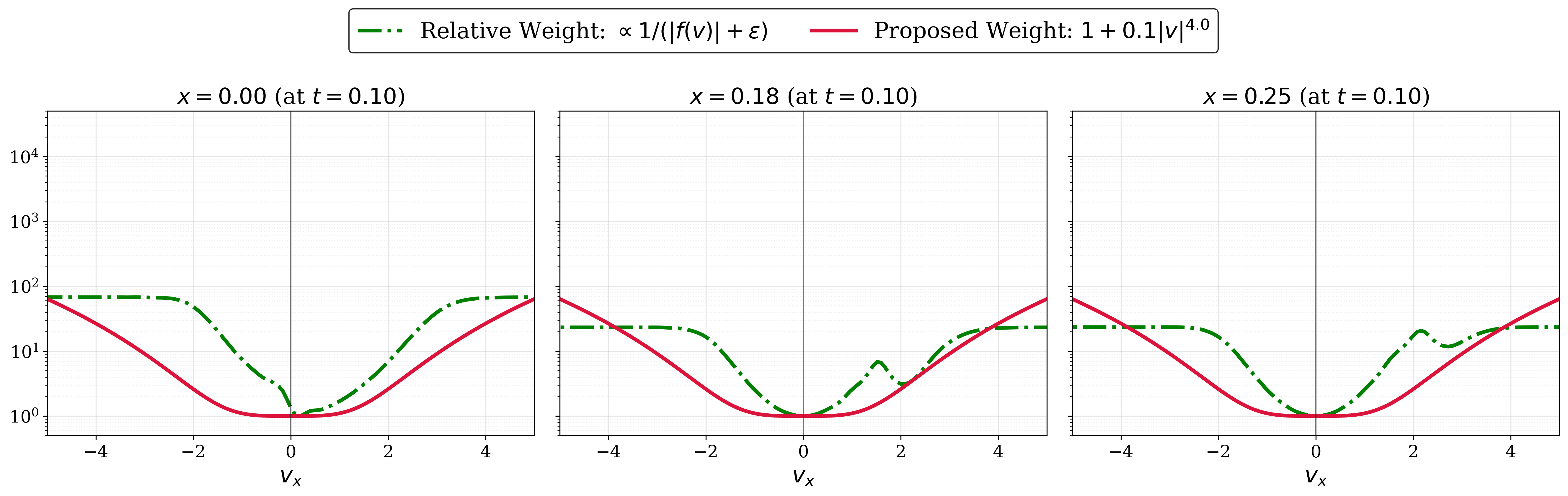}
    \caption{Comparison of weight shapes between the relative loss and the proposed loss for the 1D Riemann problem at $\mathrm{Kn}=1.0$.}
    \label{fig:weight_shape_comparison_riemann}
\end{figure}

While the relative weight exhibits such irregular profiles depending on the flow dynamics, our proposed method relies on a fixed weight function motivated by the stability analysis. As demonstrated across the numerical experiments, this approach provides reliable and consistent performance in all tested scenarios.

\section{Conclusion}
\label{sec:conclusion}

In this paper, we identified the fundamental limitations of the standard $L^2$-based PINN loss function for solving the BGK model and proposed a theory-guided weighted PINN loss to overcome these shortcomings. By constructing explicit counterexamples, we demonstrated that simply minimizing the standard $L^2$ PDE residual is insufficient to control the error in the high-velocity region, potentially leading to severe distortions in the macroscopic moments. To resolve this, we introduced a velocity-dependent weight function satisfying appropriate integrability conditions. Through a mathematical stability analysis, we proved that minimizing this weighted loss guarantees the $L^2$ convergence of the approximate solution and the $L^1$ convergence of the macroscopic quantities.

\rev{The practical efficacy of the proposed loss was validated through numerical experiments covering spatially homogeneous stress tests, smooth benchmarks in one to three spatial dimensions, and a discontinuous one-dimensional Riemann problem, over Knudsen numbers ranging from $1.0$ to $0.01$. In the homogeneous problems, the standard $L^2$ loss attains a small training loss while its error stagnates, reproducing the failure predicted by the counterexamples. Across the spatially inhomogeneous benchmarks, the proposed loss consistently attains the lowest mean errors, outperforming both the standard $L^2$ loss and the empirical relative loss, and it adds no computational overhead.}

The theoretical convergence guarantees established in this study rely on the assumption of successful optimization, which remains a universal challenge due to the non-convex nature of neural network training. Nevertheless, the theory-guided regularization strategy presented herein has proven to be a stable and effective approach for solving the BGK model. For future work, we plan to extend the proposed weighting scheme to more complex collision models, such as the full Boltzmann equation and the kinetic Fokker-Planck equation, to further verify its broad applicability.

\section*{Acknowledgment}
G. Ko was supported by the National Natural Science Foundation of China (No. 12288201). S.-J. Son was supported by the National Research Foundation of Korea (NRF) grant funded by the Korea government (MSIT and MOE) (No. RS-2023-00219980 and No. RS-2025-25419038). 
S. Y. Cho was supported by the National Research Foundation of Korea (NRF) grant funded by the Korea government (MSIT) (RS-2026-25493252). M.-S. Lee was supported by Basic Science Research Programs through the National Research Foundation of Korea (NRF) funded by the Ministry of Education (RS-2023-00244475 and RS-2024-00462755). The authors are grateful to Jaemin Oh for providing the Separable PINN implementation from~\cite{oh2025separable}, which greatly facilitated the numerical simulations in this work.

\section*{Author Contributions}
\textbf{Gyounghun Ko:} Formal analysis, Validation, Writing -- original draft.
\textbf{Sung-Jun Son:} Formal analysis, Writing -- original draft. 
\textbf{Seung Yeon Cho:} Data curation, Validation, Writing -- original draft, Writing -- review \& editing, Supervision.
\textbf{Myeong-Su Lee:} Conceptualization, Formal analysis, Software, Validation, Visualization, Writing -- original draft, Writing -- review \& editing, Supervision.

\bibliographystyle{abbrv}
\bibliography{../references}

\begin{thebibliography}{10}

\bibitem{abdo2024error}
E.~Abdo, L.~Chai, R.~Hu, and X.~Yang.
\newblock Error estimates of physics-informed neural networks for approximating
  {B}oltzmann equation.
\newblock {\em arXiv preprint arXiv:2407.08383}, 2024.

\bibitem{bhatnagar1954model}
P.~L. Bhatnagar, E.~P. Gross, and M.~Krook.
\newblock A model for collision processes in gases. {I}. {S}mall amplitude
  processes in charged and neutral one-component systems.
\newblock {\em Physical Review}, 94(3):511, 1954.

\bibitem{bird1994molecular}
G.~A. Bird.
\newblock {\em Molecular gas dynamics and the direct simulation of gas flows}.
\newblock Oxford University Press, 1994.

\bibitem{boscarino2020high}
S.~Boscarino, S.-Y. Cho, G.~Russo, and S.-B. Yun.
\newblock High order conservative {S}emi-{L}agrangian scheme for the {BGK}
  model of the {B}oltzmann equation.
\newblock {\em Communications in Computational Physics}, 29(1):1--56, 2021.

\bibitem{cai2021physics}
S.~Cai, Z.~Mao, Z.~Wang, M.~Yin, and G.~E. Karniadakis.
\newblock Physics-informed neural networks ({PINNs}) for fluid mechanics: A
  review.
\newblock {\em Acta Mechanica Sinica}, 37(12):1727--1738, 2021.

\bibitem{cercignani1988boltzmann}
C.~Cercignani.
\newblock The {B}oltzmann equation.
\newblock In {\em The Boltzmann equation and its applications}, pages 40--103.
  Springer, 1988.

\bibitem{chen2026micro}
W.~Chen, G.~Dimarco, and L.~Pareschi.
\newblock Micro-macro tensor neural surrogates for uncertainty quantification
  in collisional plasma.
\newblock {\em Journal of Computational Physics}, page 115197, 2026.
\newblock To appear.

\bibitem{chen2023symbolic}
X.~Chen, C.~Liang, D.~Huang, E.~Real, K.~Wang, H.~Pham, X.~Dong, T.~Luong,
  C.-J. Hsieh, Y.~Lu, et~al.
\newblock Symbolic discovery of optimization algorithms.
\newblock {\em Advances in Neural Information Processing Systems},
  36:49205--49233, 2023.

\bibitem{cho2023separable}
J.~Cho, S.~Nam, H.~Yang, S.-B. Yun, Y.~Hong, and E.~Park.
\newblock Separable physics-informed neural networks.
\newblock {\em Advances in Neural Information Processing Systems},
  36:23761--23788, 2023.

\bibitem{cho2021conservative2}
S.~Y. Cho, S.~Boscarino, G.~Russo, and S.-B. Yun.
\newblock Conservative semi-{L}agrangian schemes for kinetic equations {P}art
  {II}: Applications.
\newblock {\em Journal of Computational Physics}, 436:110281, 2021.

\bibitem{cho2025kinetic}
S.~Y. Cho, Y.-P. Choi, B.-H. Hwang, and S.~Song.
\newblock From kinetic mixtures to compressible two-phase flow: A {BGK}-type
  model and rigorous derivation.
\newblock {\em Kinetic and Related Models}, 22(0):147--196, 2026.

\bibitem{cho2024conservative}
S.-Y. Cho, M.~Groppi, J.-M. Qiu, G.~Russo, and S.-B. Yun.
\newblock Conservative {S}emi-{L}agrangian methods for kinetic equations.
\newblock In {\em Active Particles, Volume 4: Theory, Models, Applications},
  pages 283--420. Springer, 2024.

\bibitem{de2024numerical}
T.~De~Ryck and S.~Mishra.
\newblock Numerical analysis of physics-informed neural networks and related
  models in physics-informed machine learning.
\newblock {\em Acta Numerica}, 33:633--713, 2024.

\bibitem{dimarco2014numerical}
G.~Dimarco and L.~Pareschi.
\newblock Numerical methods for kinetic equations.
\newblock {\em Acta Numerica}, 23:369--520, 2014.

\bibitem{e2022algorithms}
W.~E, J.~Han, and A.~Jentzen.
\newblock Algorithms for solving high dimensional {PDEs}: From nonlinear {Monte
  Carlo} to machine learning.
\newblock {\em Nonlinearity}, 35(1):278--310, 2022.

\bibitem{e2018deep}
W.~E and B.~Yu.
\newblock The deep {Ritz} method: A deep learning-based numerical algorithm for
  solving variational problems.
\newblock {\em Communications in Mathematics and Statistics}, 6(1):1--12, 2018.

\bibitem{evans2022partial}
L.~C. Evans.
\newblock {\em Partial differential equations}, volume~19.
\newblock American Mathematical Society, 2022.

\bibitem{filbet2010class}
F.~Filbet and S.~Jin.
\newblock A class of asymptotic-preserving schemes for kinetic equations and
  related problems with stiff sources.
\newblock {\em Journal of Computational Physics}, 229(20):7625--7648, 2010.

\bibitem{filbet2006solving}
F.~Filbet, C.~Mouhot, and L.~Pareschi.
\newblock Solving the {Boltzmann} equation in {$N\log_2N$}.
\newblock {\em SIAM Journal on Scientific Computing}, 28(3):1029--1053, 2006.

\bibitem{gamba2017fast}
I.~M. Gamba, J.~R. Haack, C.~D. Hauck, and J.~Hu.
\newblock A fast spectral method for the {Boltzmann} collision operator with
  general collision kernels.
\newblock {\em SIAM Journal on Scientific Computing}, 39(4):B658--B674, 2017.

\bibitem{gie2024semi}
G.-M. Gie, Y.~Hong, and C.-Y. Jung.
\newblock {S}emi-analytic {PINN} methods for singularly perturbed boundary
  value problems.
\newblock {\em Applicable Analysis}, 103(14):2554--2571, 2024.

\bibitem{gilbarg1998elliptic}
D.~Gilbarg and N.~S. Trudinger.
\newblock {\em Elliptic Partial Differential Equations of Second Order}.
\newblock Classics in Mathematics. Springer, 2001.

\bibitem{han2018solving}
J.~Han, A.~Jentzen, and W.~E.
\newblock Solving high-dimensional partial differential equations using deep
  learning.
\newblock {\em Proceedings of the National Academy of Sciences},
  115(34):8505--8510, 2018.

\bibitem{hu2018asymptotic}
J.~Hu, R.~Shu, and X.~Zhang.
\newblock Asymptotic-preserving and positivity-preserving implicit-explicit
  schemes for the stiff {BGK} equation.
\newblock {\em SIAM Journal on Numerical Analysis}, 56(2):942--973, 2018.

\bibitem{hwang2020trend}
H.~J. Hwang, J.~W. Jang, H.~Jo, and J.~Y. Lee.
\newblock Trend to equilibrium for the kinetic {F}okker-{P}lanck equation via
  the neural network approach.
\newblock {\em Journal of Computational Physics}, 419:109665, 2020.

\bibitem{jin2024asymptotic2}
S.~Jin, Z.~Ma, and K.~Wu.
\newblock Asymptotic-preserving neural networks for multiscale kinetic
  equations.
\newblock {\em Communications in Computational Physics}, 35(3):693--723, 2024.

\bibitem{jin2024asymptotic}
S.~Jin, Z.~Ma, and T.-a. Zhang.
\newblock Asymptotic-preserving neural networks for multiscale
  {V}lasov--{P}oisson--{F}okker--{P}lanck system in the high-field regime.
\newblock {\em Journal of Scientific Computing}, 99(3):61, 2024.

\bibitem{jo2026task}
S.~Jo, S.~Park, J.~Shin, J.~Park, H.~Kim, S.~Ko, S.~Lee, and J.~Jeon.
\newblock Task-aware evolution in physics-informed neural networks: Application
  to {S}aint-{V}enant torsion problems.
\newblock {\em Engineering Applications of Artificial Intelligence},
  168:113988, 2026.

\bibitem{karniadakis2021physics}
G.~E. Karniadakis, I.~G. Kevrekidis, L.~Lu, P.~Perdikaris, S.~Wang, and
  L.~Yang.
\newblock Physics-informed machine learning.
\newblock {\em Nature Reviews Physics}, 3(6):422--440, 2021.

\bibitem{kim2021stationary}
D.~Kim, M.-S. Lee, and S.-B. Yun.
\newblock Stationary {BGK} models for chemically reacting gas in a slab.
\newblock {\em Journal of Statistical Physics}, 184(2):24, 2021.

\bibitem{kim2026physics}
M.~Kim, S.-J. Son, Y.~Kim, and D.~Lee.
\newblock A physics-informed, global-in-time neural particle method for the
  spatially homogeneous {L}andau equation.
\newblock {\em arXiv preprint arXiv:2603.10874}, 2026.

\bibitem{kim2024physics}
S.~Kim, S.-B. Yun, H.-O. Bae, M.~Lee, and Y.~Hong.
\newblock Physics-informed convolutional transformer for predicting volatility
  surface.
\newblock {\em Quantitative Finance}, 24(2):203--220, 2024.

\bibitem{kingma2014adam}
D.~P. Kingma and J.~Ba.
\newblock {A}dam: A method for stochastic optimization.
\newblock {\em arXiv preprint arXiv:1412.6980}, 2014.

\bibitem{krishnapriyan2021characterizing}
A.~S. Krishnapriyan, A.~Gholami, S.~Zhe, R.~M. Kirby, and M.~W. Mahoney.
\newblock Characterizing possible failure modes in physics-informed neural
  networks.
\newblock In {\em Advances in Neural Information Processing Systems},
  volume~34, pages 26548--26560, 2021.

\bibitem{lee2023oppinn}
J.~Y. Lee, J.~Jang, and H.~J. Hwang.
\newblock {opPINN}: Physics-informed neural network with operator learning to
  approximate solutions to the {F}okker-{P}lanck-{L}andau equation.
\newblock {\em Journal of Computational Physics}, 480:112031, 2023.

\bibitem{lee2025forward}
M.-S. Lee, J.~Oh, D.-C. Lee, K.~Lee, S.~Park, and Y.~Hong.
\newblock Forward and inverse simulation of pseudo-two-dimensional model of
  lithium-ion batteries using neural networks.
\newblock {\em Computer Methods in Applied Mechanics and Engineering},
  438:117856, 2025.

\bibitem{li2021physics}
R.~Li, E.~Lee, and T.~Luo.
\newblock Physics-informed neural networks for solving multiscale mode-resolved
  phonon {B}oltzmann transport equation.
\newblock {\em Materials Today Physics}, 19:100429, 2021.

\bibitem{li2024solving}
Z.~Li, Y.~Wang, H.~Liu, Z.~Wang, and B.~Dong.
\newblock Solving the {B}oltzmann equation with a neural sparse representation.
\newblock {\em SIAM Journal on Scientific Computing}, 46(2):C186--C215, 2024.

\bibitem{lou2021physics}
Q.~Lou, X.~Meng, and G.~E. Karniadakis.
\newblock Physics-informed neural networks for solving forward and inverse flow
  problems via the {Boltzmann-BGK} formulation.
\newblock {\em Journal of Computational Physics}, 447:110676, 2021.

\bibitem{mieussens2000discrete}
L.~Mieussens.
\newblock Discrete velocity model and implicit scheme for the {BGK} equation of
  rarefied gas dynamics.
\newblock {\em Mathematical Models and Methods in Applied Sciences},
  10(08):1121--1149, 2000.

\bibitem{oh2025separable}
J.~Oh, S.~Y. Cho, S.-B. Yun, E.~Park, and Y.~Hong.
\newblock Separable physics-informed neural networks for solving the {BGK}
  model of the {B}oltzmann equation.
\newblock {\em SIAM Journal on Scientific Computing}, 47(2):C451--C474, 2025.

\bibitem{pang2019fpinns}
G.~Pang, L.~Lu, and G.~E. Karniadakis.
\newblock {fPINNs}: Fractional physics-informed neural networks.
\newblock {\em SIAM Journal on Scientific Computing}, 41(4):A2603--A2626, 2019.

\bibitem{pieraccini2007implicit}
S.~Pieraccini and G.~Puppo.
\newblock Implicit--explicit schemes for {BGK} kinetic equations.
\newblock {\em Journal of Scientific Computing}, 32(1):1--28, 2007.

\bibitem{raissi2019physics}
M.~Raissi, P.~Perdikaris, and G.~E. Karniadakis.
\newblock Physics-informed neural networks: A deep learning framework for
  solving forward and inverse problems involving nonlinear partial differential
  equations.
\newblock {\em Journal of Computational Physics}, 378:686--707, 2019.

\bibitem{saint2002modele}
L.~Saint-Raymond.
\newblock Du mod{\`e}le {BGK} de l'{\'e}quation de {B}oltzmann aux
  {\'e}quations d'{E}uler des fluides incompressibles.
\newblock {\em Bulletin des sciences mathematiques}, 126(6):493--506, 2002.

\bibitem{sands2025adaptive}
W.~A. Sands, J.-M. Qiu, D.~Hayes, and N.~Zheng.
\newblock An adaptive-rank approach with greedy sampling for multi-scale {BGK}
  equations.
\newblock {\em Journal of Computational Physics}, 547:114523, 2026.

\bibitem{sitzmann2020implicit}
V.~Sitzmann, J.~N.~P. Martel, A.~W. Bergman, D.~B. Lindell, and G.~Wetzstein.
\newblock Implicit neural representations with periodic activation functions.
\newblock In {\em Advances in Neural Information Processing Systems (NeurIPS)},
  volume~33, pages 7462--7473, 2020.

\bibitem{Son2025}
S.-J. Son.
\newblock {$L^P$}-solutions to the {ES}-{BGK} model of the polyatomic
  molecules.
\newblock {\em Journal of Mathematical Physics}, 65(10):101501, 2024.

\bibitem{SY2022}
S.-j. Son and S.-B. Yun.
\newblock {C}auchy problem for the {ES}-{BGK} model with the correct {P}randtl
  number.
\newblock {\em Partial Differential Equations and Applications}, 3(3):41, 2022.

\bibitem{SY2023}
S.-j. Son and S.-B. Yun.
\newblock The {ES}-{BGK} for the polyatomic molecules with infinite energy.
\newblock {\em Journal of Statistical Physics}, 190(8):129, 2023.

\bibitem{SY2025}
S.-J. Son and S.-B. Yun.
\newblock Local in time solution to {ES}-{BGK} model with correct {P}randtl
  number.
\newblock {\em Kinetic and Related Models}, 18(4):499--519, 2025.

\bibitem{sone2007molecular}
Y.~Sone.
\newblock {\em Molecular gas dynamics: theory, techniques, and applications}.
\newblock Springer, 2007.

\bibitem{wang20222}
C.~Wang, S.~Li, D.~He, and L.~Wang.
\newblock Is {$L^2$} physics informed loss always suitable for training physics
  informed neural network?
\newblock {\em Advances in Neural Information Processing Systems},
  35:8278--8290, 2022.

\bibitem{wang2021understanding}
S.~Wang, Y.~Teng, and P.~Perdikaris.
\newblock Understanding and mitigating gradient flow pathologies in
  physics-informed neural networks.
\newblock {\em SIAM Journal on Scientific Computing}, 43(5):A3055--A3081, 2021.

\bibitem{yun2010cauchy}
S.-B. Yun.
\newblock {C}auchy problem for the {B}oltzmann-{BGK} model near a global
  {M}axwellian.
\newblock {\em Journal of Mathematical Physics}, 51(12), 2010.

\bibitem{yun2015ellipsoidal}
S.-B. Yun.
\newblock Ellipsoidal {BGK} model near a global {M}axwellian.
\newblock {\em SIAM Journal on Mathematical Analysis}, 47(3):2324--2354, 2015.

\bibitem{Yun2019}
S.-B. Yun.
\newblock Ellipsoidal {BGK} model for polyatomic molecules near {M}axwellians:
  a dichotomy in the dissipation estimate.
\newblock {\em Journal of Differential Equations}, 266(9):5566--5614, 2019.

\bibitem{zhang2023physics}
B.~Zhang, G.~Cai, H.~Weng, W.~Wang, L.~Liu, and B.~He.
\newblock Physics-informed neural networks for solving forward and inverse
  {V}lasov--{P}oisson equation via fully kinetic simulation.
\newblock {\em Machine Learning: Science and Technology}, 4(4):045015, 2023.

\end{thebibliography}

\appendix
\section{Implementation details and reproducibility}
\label{app:repro}

\rev{This section details the implementation and reference-solver settings used in the numerical experiments. The coordinate networks introduced in Subsection~\ref{subsec:experimental_setup} use the SIREN frequency $\omega_0 = 10$ and tensor rank $R = 256$, and their weights are initialized with the standard SIREN initialization scheme~\cite{sitzmann2020implicit}. The Gaussian velocity factor is initialized with mean $\mu = \mathbf{0}$ and precision $\tau = 1$, both of which are trained. The number of hidden layers and their width are chosen per problem, as listed in Table~\ref{tab:hyperparams}. The network is optimized with the Lion optimizer~\cite{chen2023symbolic} under a cosine-decay learning-rate schedule decaying to zero, and the collocation points are resampled uniformly at every iteration. The composite objective is $\mathcal{L} = \mathcal{L}_{pde} + \lambda_{ini}\,\mathcal{L}_{ini} + \lambda_{bc}\,\mathcal{L}_{bc}$ with $\lambda_{bc} = 1$, where the boundary term is absent for the spatially homogeneous problem and $\lambda_{ini}$ is problem-specific. Macroscopic moments are evaluated by the trapezoidal rule with $257$ uniform grid points per velocity axis. All simulations are run over the time interval $t \in (0, 0.1]$. Unless stated otherwise, each reported result is the mean and standard deviation over five independent training runs that differ only in the random seed, and three runs for the three-dimensional problem.}

\begin{table}[htbp]
\centering
\caption{Problem-specific hyperparameters for all benchmarks. Here $L\times W$ is the number of hidden layers times their width in each coordinate network, $N_\bullet$ is the number of collocation points sampled per iteration along each coordinate axis, $N_v$ is the velocity-axis count, written as a triple $(N_{v_x},N_{v_y},N_{v_z})$ when the three axes differ, $[-V,V]^3$ is the truncated velocity domain, and $\lambda_{ini}$ is the initial-condition loss weight.}
\label{tab:hyperparams}
\footnotesize
\setlength{\tabcolsep}{3pt}%
\begin{tabular}{lcccccc}
\toprule
\textbf{Problem} & $\boldsymbol{L\times W}$ & $\boldsymbol{[-V,V]^3}$ & \textbf{Collocation points} & \textbf{Iterations} & \textbf{Init.\ LR} & $\boldsymbol{\lambda_{ini}}$ \\
\midrule
Homogeneous & $4\times64$  & $[-8,8]^3$   & $N_t{=}24,\, N_v{=}(32,24,24)$ & $10^5$ & $10^{-5}$ & $10^{2}$ \\
1D smooth   & $2\times128$ & $[-10,10]^3$ & $N_t{=}12,\, N_x{=}16,\, N_v{=}12$ & $10^5$ & $10^{-5}$ & $10^{3}$ \\
1D Riemann  & $2\times128$ & $[-8,8]^3$   & $N_t{=}12,\, N_x{=}32,\, N_v{=}(32,12,12)$ & $2\times10^5$ & $10^{-5}$ & $10^{3}$ \\
2D smooth   & $2\times128$ & $[-10,10]^3$ & $N_t{=}12,\, N_x{=}N_y{=}16,\, N_v{=}12$ & $10^5$ & $10^{-5}$ & $10^{3}$ \\
3D smooth   & $2\times128$ & $[-8,8]^3$   & $N_t{=}N_x{=}N_y{=}N_z{=}12,\, N_v{=}12$ & $10^5$ & $10^{-5}$ & $10^{3}$ \\
\bottomrule
\end{tabular}
\end{table}

\rev{Reference solutions for the spatially inhomogeneous problems are computed with the high-order conservative semi-Lagrangian scheme of~\cite{cho2021conservative2, boscarino2020high}. The spatial mesh uses $N_x = 1280$ for the one-dimensional problems, $N_x = N_y = 160$ for the two-dimensional problem, and $N_x = 80$ for the three-dimensional problem, and the velocity space is discretized with $96 \times 24 \times 24$ cells on $[-10, 10]^3$. The errors of the spatially inhomogeneous benchmarks are measured at the terminal time $t = 0.1$ using the network parameters at the end of training. For the spatially homogeneous problems the exact solution is available in closed form, as given in~\eqref{eq:homog_exact}.}

\section{Additional homogeneous configurations}
\label{app:homog-extra}

\rev{This appendix reports the spatially homogeneous families of Subsection~\ref{subsec:homogeneous} at additional parameter values, namely $T_h \in \{4, 5\}$ for Case A, $H_x \in \{3, 5\}$ for Case B, and $v_0 \in \{2, 3\}$ for Case C. Table~\ref{tab:homog_extra} reports the final training loss and the relative $L^2$ errors, and Figure~\ref{fig:homog_extra} shows the corresponding relation between the training loss and the error of $f$. The failure of the standard loss deepens as the suprathermal component becomes more energetic, while the weighted loss remains accurate across all configurations.}

\begin{table}[htbp]
\centering
\caption{\rev{Final training loss and relative $L^2$ errors of $f$ and $T$ for the additional homogeneous configurations. Values are the mean $\pm$ standard deviation over five runs.}}
\label{tab:homog_extra}
\resizebox{\textwidth}{!}{%
\setlength{\tabcolsep}{6pt}%
\begin{tabular}{lcccccc}
\toprule
& \multicolumn{3}{c}{\textbf{Standard $L^2$ Loss}} & \multicolumn{3}{c}{\textbf{Proposed Loss}} \\
\cmidrule(lr){2-4} \cmidrule(lr){5-7}
\textbf{Initial datum} & \textbf{Training loss} & $\boldsymbol{f}$ \textbf{(Relative $L^2$)} & $\boldsymbol{T}$ \textbf{(Relative $L^2$)} & \textbf{Training loss} & $\boldsymbol{f}$ \textbf{(Relative $L^2$)} & $\boldsymbol{T}$ \textbf{(Relative $L^2$)} \\
\midrule
Case A ($T_h=4$) & $(1.02 \pm 0.07) \times 10^{-9}$ & $(2.68 \pm 0.08) \times 10^{-2}$ & $(7.86 \pm 0.22) \times 10^{-2}$ & $(1.22 \pm 0.16) \times 10^{-8}$ & $\mathbf{(3.18 \pm 0.44) \times 10^{-4}}$ & $\mathbf{(8.08 \pm 0.42) \times 10^{-4}}$ \\  
Case A ($T_h=5$) & $(2.61 \pm 0.11) \times 10^{-9}$ & $(4.74 \pm 0.09) \times 10^{-2}$ & $(1.39 \pm 0.02) \times 10^{-1}$ & $(2.58 \pm 0.08) \times 10^{-8}$ & $\mathbf{(3.70 \pm 0.51) \times 10^{-4}}$ & $\mathbf{(7.61 \pm 1.44) \times 10^{-4}}$ \\  
Case B ($H_x=3$) & $(1.64 \pm 0.22) \times 10^{-11}$ & $(8.69 \pm 0.51) \times 10^{-4}$ & $(2.24 \pm 0.07) \times 10^{-3}$ & $(7.16 \pm 0.80) \times 10^{-11}$ & $\mathbf{(8.57 \pm 1.03) \times 10^{-5}}$ & $\mathbf{(3.11 \pm 0.81) \times 10^{-5}}$ \\  
Case B ($H_x=5$) & $(1.37 \pm 0.17) \times 10^{-10}$ & $(3.15 \pm 0.26) \times 10^{-3}$ & $(8.65 \pm 0.74) \times 10^{-3}$ & $(3.42 \pm 0.55) \times 10^{-10}$ & $\mathbf{(1.09 \pm 0.14) \times 10^{-4}}$ & $\mathbf{(6.51 \pm 1.79) \times 10^{-5}}$ \\  
Case C ($v_0=2$) & $(4.23 \pm 0.59) \times 10^{-12}$ & $(1.69 \pm 0.30) \times 10^{-4}$ & $(4.72 \pm 0.35) \times 10^{-4}$ & $(2.14 \pm 0.31) \times 10^{-11}$ & $\mathbf{(7.88 \pm 0.55) \times 10^{-5}}$ & $\mathbf{(2.71 \pm 0.67) \times 10^{-5}}$ \\  
Case C ($v_0=3$) & $(3.22 \pm 0.45) \times 10^{-11}$ & $(8.86 \pm 0.70) \times 10^{-4}$ & $(2.87 \pm 0.25) \times 10^{-3}$ & $(1.42 \pm 0.12) \times 10^{-10}$ & $\mathbf{(8.20 \pm 1.53) \times 10^{-5}}$ & $\mathbf{(4.07 \pm 2.01) \times 10^{-5}}$ \\  
\bottomrule
\end{tabular}
}
\end{table}

\begin{figure}[htbp]
    \centering
    \includegraphics[width=\textwidth]{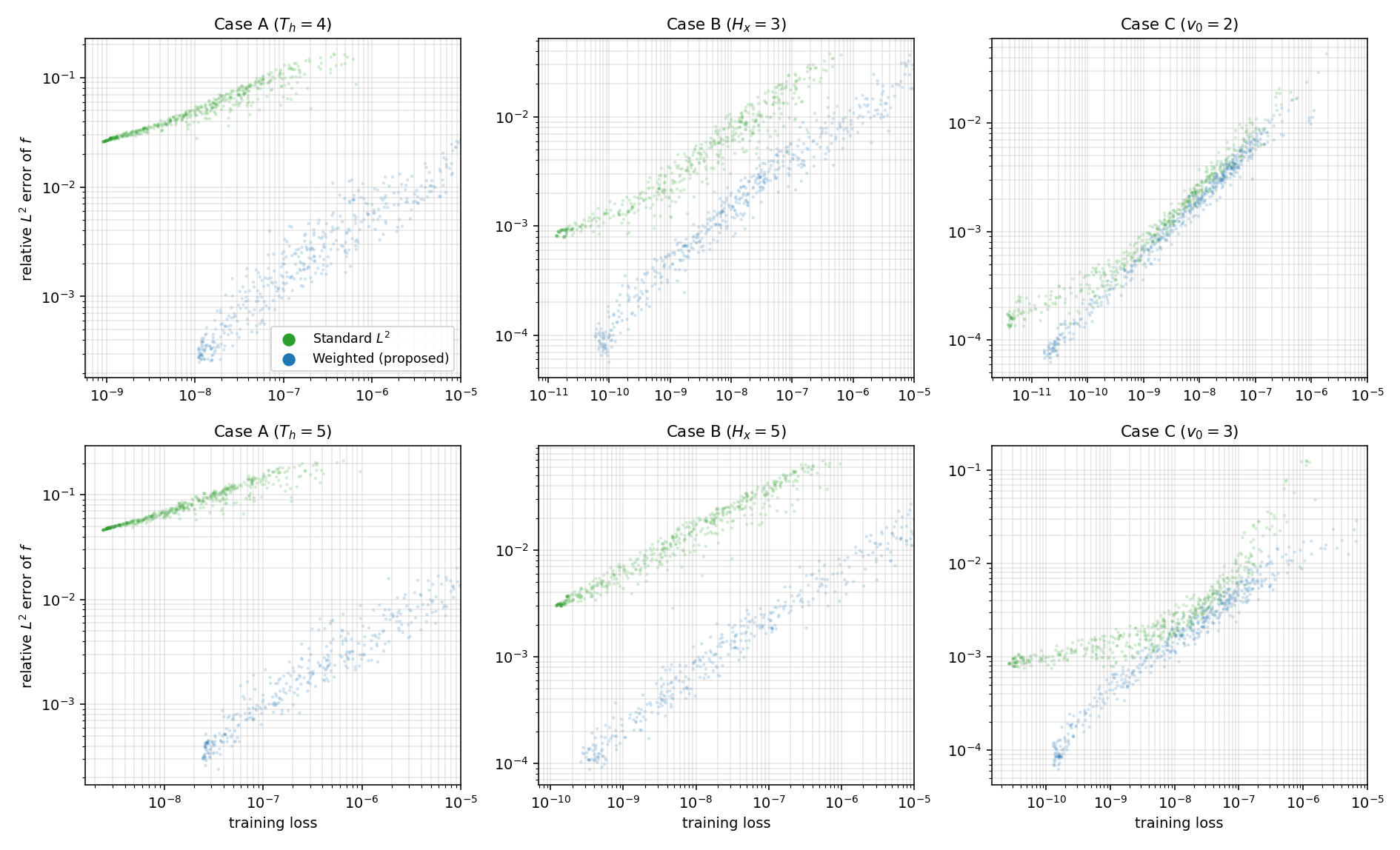}
    \caption{\rev{Training loss versus relative $L^2$ error of $f$ for the additional homogeneous configurations.}}
    \label{fig:homog_extra}
\end{figure}

\section{Ablation Study on the Weight Hyperparameters}
\label{app:ablation}

For the polynomial weight function of the form $w(v) = 1 + \alpha |v|^\beta$, Theorem~\ref{thm:main_stability} guarantees that any combination of $\alpha > 0$ and $\beta > 7/2$ ensures the accuracy of the approximate solution, provided that the loss function is sufficiently minimized. However, in practical implementations, the aforementioned theoretical stability estimate implicitly assumes successful optimization, which is often difficult to achieve in practice due to the strong sensitivity of the training process to the choice of hyperparameters. \rev{In this appendix, we examine how the numerical performance depends on the choice of $\alpha$ and $\beta$, and verify that the setting $(\alpha, \beta) = (0.1, 4.0)$ used throughout the paper lies in a reliable range. The ablation is conducted on the one-dimensional smooth problem of Subsection~\ref{subsec:1d_smooth}, with a single run per configuration.}

\rev{Figures~\ref{fig:ablation_curves_kn001}--\ref{fig:ablation_curves_kn1} show the relative error curves of the distribution function and the macroscopic moments with respect to the polynomial growth rate $\beta$, for several values of $\alpha$, at $\mathrm{Kn} = 0.01, 0.1$, and $1.0$. The horizontal lines mark the two baselines.}
\begin{figure}[htbp]
    \centering
    \includegraphics[width=\textwidth]{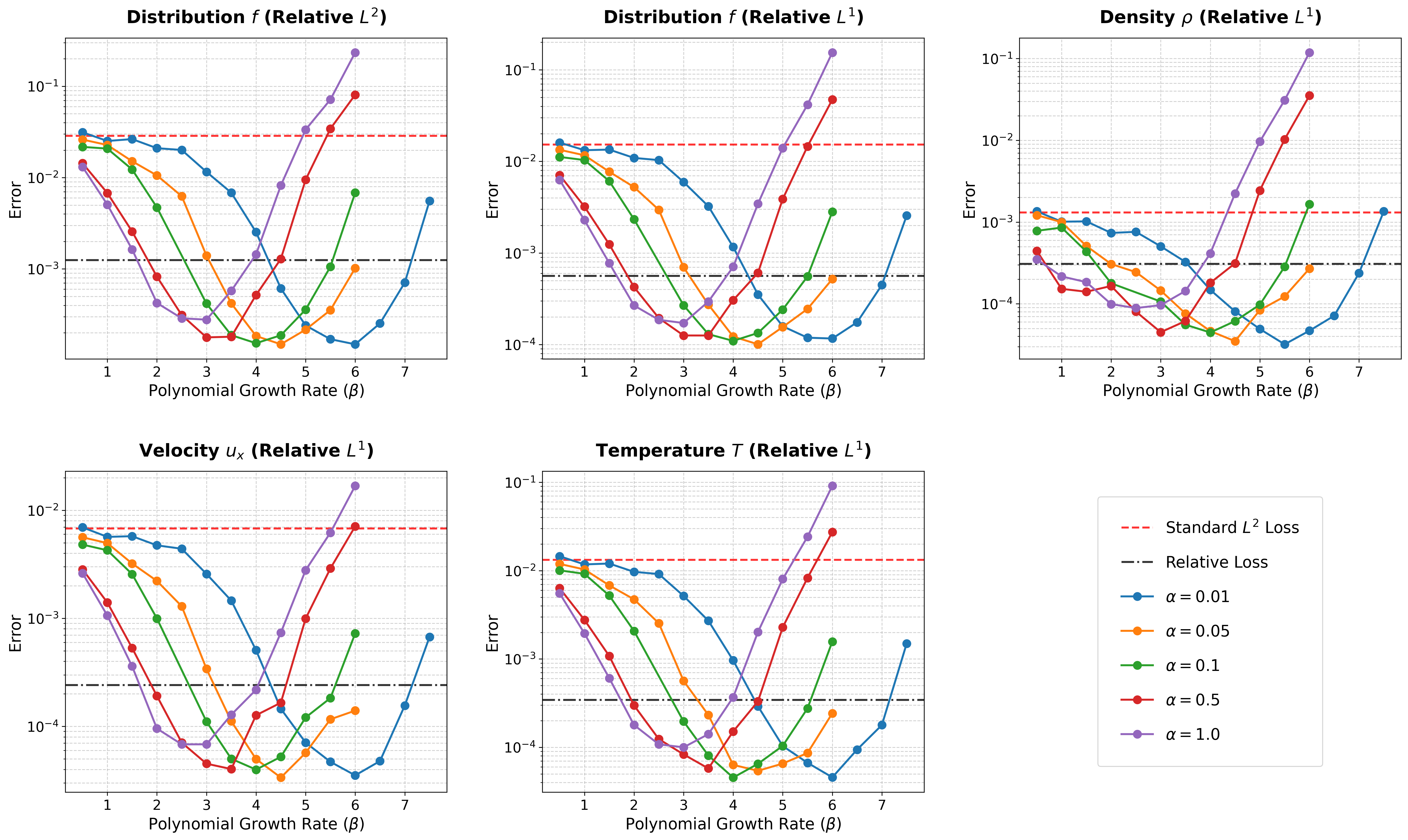}
    \caption{Relative error curves of the distribution function $f$ and macroscopic moments ($\rho, u_x, T$) according to the polynomial growth rate $\beta$ at $\mathrm{Kn} = 0.01$.}
    \label{fig:ablation_curves_kn001}
\end{figure}

\begin{figure}[htbp]
    \centering
    \includegraphics[width=\textwidth]{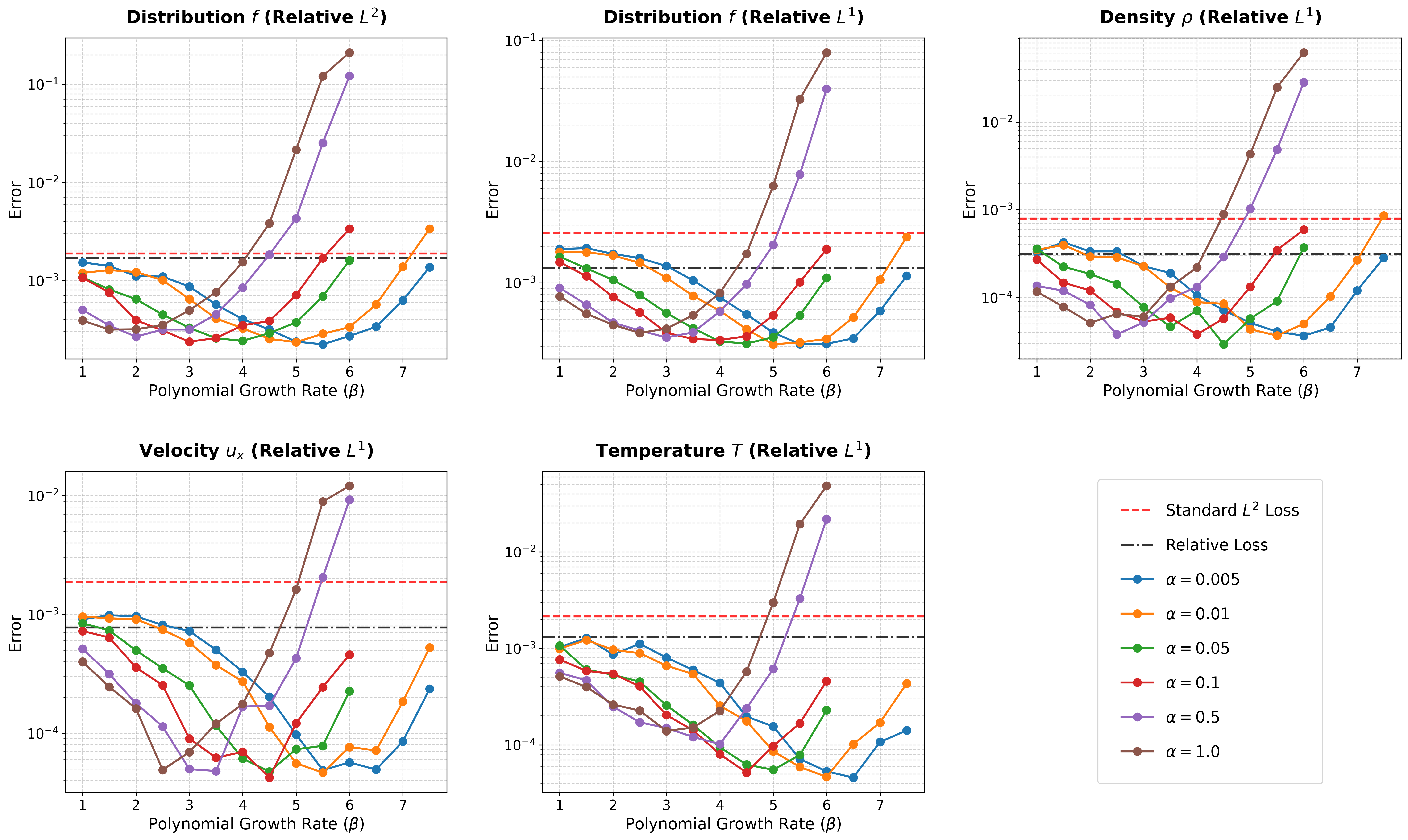}
    \caption{Relative error curves of the distribution function $f$ and macroscopic moments ($\rho, u_x, T$) according to the polynomial growth rate $\beta$ at $\mathrm{Kn} = 0.1$.}
    \label{fig:ablation_curves_kn01}
\end{figure}

\begin{figure}[htbp]
    \centering
    \includegraphics[width=\textwidth]{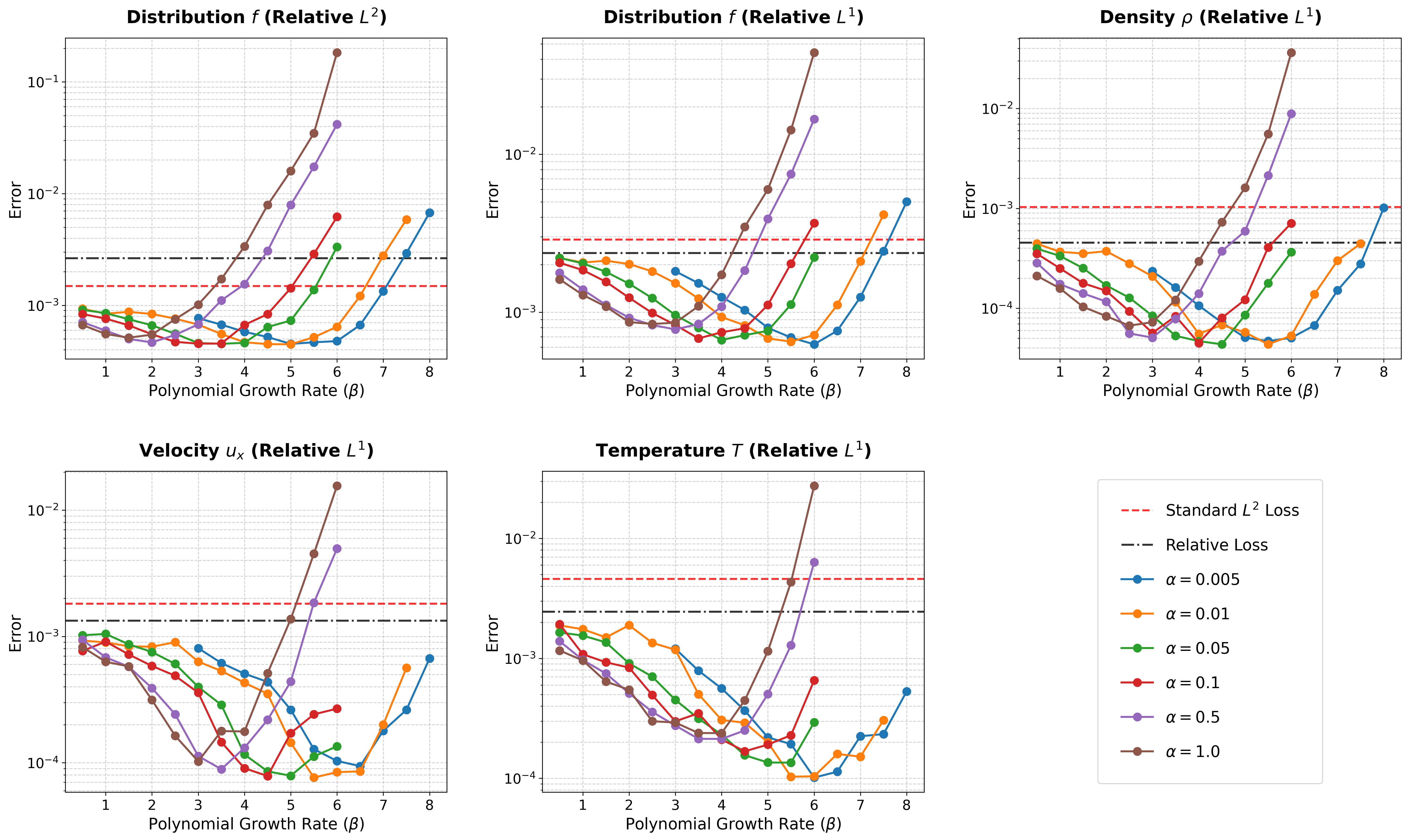}
    \caption{Relative error curves of the distribution function $f$ and macroscopic moments ($\rho, u_x, T$) according to the polynomial growth rate $\beta$ at $\mathrm{Kn} = 1.0$.}
    \label{fig:ablation_curves_kn1}
\end{figure}

\rev{Table~\ref{tab:ablation_kn001} shows the specific error values at $\mathrm{Kn} = 0.01$ for the baselines and the cases fixed at $\alpha = 0.1$.}
\begin{table}[htbp]
\centering
\caption{Ablation study results of relative errors at $\mathrm{Kn} = 0.01$ for the standard $L^2$ Loss, relative loss~\cite{oh2025separable}, and the proposed loss with specific choices of $\alpha$ and $\beta$.}
\label{tab:ablation_kn001}
\resizebox{\textwidth}{!}{%
\setlength{\tabcolsep}{16.1pt}%
\begin{tabular}{lccccc}
\toprule
\textbf{Loss function} & $\boldsymbol{f}$ \textbf{(Relative $L^2$)} & $\boldsymbol{f}$ \textbf{(Relative $L^1$)} & $\boldsymbol{\rho}$ \textbf{(Relative $L^1$)} & $\boldsymbol{u_x}$ \textbf{(Relative $L^1$)} & $\boldsymbol{T}$ \textbf{(Relative $L^1$)} \\
\midrule
Standard $L^2$ Loss      & $2.87 \times 10^{-2}$ & $1.53 \times 10^{-2}$ & $1.31 \times 10^{-3}$ & $6.78 \times 10^{-3}$ & $1.33 \times 10^{-2}$ \\
Relative Loss~\cite{oh2025separable}  & $1.25 \times 10^{-3}$ & $5.63 \times 10^{-4}$ & $3.09 \times 10^{-4}$ & $2.42 \times 10^{-4}$ & $3.44 \times 10^{-4}$ \\
\midrule
$\alpha=0.1, \beta=2.0$ & $4.71 \times 10^{-3}$ & $2.33 \times 10^{-3}$ & $1.78 \times 10^{-4}$ & $9.97 \times 10^{-4}$ & $2.08 \times 10^{-3}$ \\
$\alpha=0.1, \beta=3.0$ & $4.17 \times 10^{-4}$ & $2.68 \times 10^{-4}$ & $1.07 \times 10^{-4}$ & $1.11 \times 10^{-4}$ & $1.96 \times 10^{-4}$ \\
$\alpha=0.1, \beta=3.5$ & $1.88 \times 10^{-4}$ & $1.30 \times 10^{-4}$ & $5.55 \times 10^{-5}$ & $5.02 \times 10^{-5}$ & $8.06 \times 10^{-5}$ \\
{\boldmath$\alpha=0.1,\ \beta=4.0$} & $\mathbf{1.54 \times 10^{-4}}$ & $\mathbf{1.10 \times 10^{-4}}$ & $\mathbf{4.46 \times 10^{-5}}$ & $\mathbf{3.98 \times 10^{-5}}$ & $\mathbf{4.54 \times 10^{-5}}$ \\
$\alpha=0.1, \beta=5.0$ & $3.59 \times 10^{-4}$ & $2.41 \times 10^{-4}$ & $9.76 \times 10^{-5}$ & $1.21 \times 10^{-4}$ & $1.04 \times 10^{-4}$ \\
$\alpha=0.1, \beta=6.0$ & $6.83 \times 10^{-3}$ & $2.82 \times 10^{-3}$ & $1.65 \times 10^{-3}$ & $7.25 \times 10^{-4}$ & $1.57 \times 10^{-3}$ \\
\bottomrule
\end{tabular}
}
\end{table}

\rev{Tables~\ref{tab:ablation_kn01} and~\ref{tab:ablation_kn1} report the corresponding values at $\mathrm{Kn} = 0.1$ and $1.0$, including additional scaling factors $\alpha \in \{0.01, 0.05\}$.}
\begin{table}[htbp]
\centering
\caption{Ablation study results of relative errors at $\mathrm{Kn} = 0.1$ for the standard $L^2$ Loss, relative loss~\cite{oh2025separable}, and the proposed loss with specific choices of $\alpha$ and $\beta$.}
\label{tab:ablation_kn01}
\resizebox{\textwidth}{!}{%
\setlength{\tabcolsep}{16.1pt}%
\begin{tabular}{lccccc}
\toprule
\textbf{Loss function} & $\boldsymbol{f}$ \textbf{(Relative $L^2$)} & $\boldsymbol{f}$ \textbf{(Relative $L^1$)} & $\boldsymbol{\rho}$ \textbf{(Relative $L^1$)} & $\boldsymbol{u_x}$ \textbf{(Relative $L^1$)} & $\boldsymbol{T}$ \textbf{(Relative $L^1$)} \\
\midrule
Standard $L^2$ Loss      & $1.88 \times 10^{-3}$ & $2.57 \times 10^{-3}$ & $7.92 \times 10^{-4}$ & $1.88 \times 10^{-3}$ & $2.14 \times 10^{-3}$ \\
Relative Loss~\cite{oh2025separable}  & $1.69 \times 10^{-3}$ & $1.33 \times 10^{-3}$ & $3.16 \times 10^{-4}$ & $7.75 \times 10^{-4}$ & $1.32 \times 10^{-3}$ \\
\midrule
$\alpha=0.01, \beta=4.0$ & $3.25 \times 10^{-4}$ & $5.87 \times 10^{-4}$ & $8.85 \times 10^{-5}$ & $2.72 \times 10^{-4}$ & $2.55 \times 10^{-4}$ \\
$\alpha=0.01, \beta=4.5$ & $2.53 \times 10^{-4}$ & $4.14 \times 10^{-4}$ & $8.45 \times 10^{-5}$ & $1.13 \times 10^{-4}$ & $1.76 \times 10^{-4}$ \\
$\alpha=0.01, \beta=5.0$ & $\mathbf{2.34 \times 10^{-4}}$ & $\mathbf{3.10 \times 10^{-4}}$ & $4.34 \times 10^{-5}$ & $5.58 \times 10^{-5}$ & $8.58 \times 10^{-5}$ \\
$\alpha=0.01, \beta=5.5$ & $2.85 \times 10^{-4}$ & $3.22 \times 10^{-4}$ & $3.67 \times 10^{-5}$ & $4.65 \times 10^{-5}$ & $5.89 \times 10^{-5}$ \\
$\alpha=0.01, \beta=6.0$ & $3.33 \times 10^{-4}$ & $3.44 \times 10^{-4}$ & $5.00 \times 10^{-5}$ & $7.65 \times 10^{-5}$ & $\mathbf{4.66 \times 10^{-5}}$ \\
\midrule
$\alpha=0.05, \beta=3.0$ & $3.28 \times 10^{-4}$ & $5.62 \times 10^{-4}$ & $7.76 \times 10^{-5}$ & $2.53 \times 10^{-4}$ & $2.57 \times 10^{-4}$ \\
$\alpha=0.05, \beta=3.5$ & $2.57 \times 10^{-4}$ & $4.21 \times 10^{-4}$ & $4.65 \times 10^{-5}$ & $1.16 \times 10^{-4}$ & $1.61 \times 10^{-4}$ \\
$\alpha=0.05, \beta=4.0$ & $2.42 \times 10^{-4}$ & $3.27 \times 10^{-4}$ & $7.07 \times 10^{-5}$ & $6.10 \times 10^{-5}$ & $9.45 \times 10^{-5}$ \\
$\alpha=0.05, \beta=4.5$ & $2.90 \times 10^{-4}$ & $3.15 \times 10^{-4}$ & $\mathbf{2.92 \times 10^{-5}}$ & $4.76 \times 10^{-5}$ & $6.25 \times 10^{-5}$ \\
$\alpha=0.05, \beta=5.0$ & $3.73 \times 10^{-4}$ & $3.55 \times 10^{-4}$ & $5.72 \times 10^{-5}$ & $7.32 \times 10^{-5}$ & $5.51 \times 10^{-5}$ \\
\midrule
$\alpha=0.1, \beta=3.0$  & $2.37 \times 10^{-4}$ & $3.87 \times 10^{-4}$ & $5.29 \times 10^{-5}$ & $9.02 \times 10^{-5}$ & $2.04 \times 10^{-4}$ \\
$\alpha=0.1, \beta=3.5$  & $2.58 \times 10^{-4}$ & $3.43 \times 10^{-4}$ & $5.88 \times 10^{-5}$ & $6.22 \times 10^{-5}$ & $1.42 \times 10^{-4}$ \\
$\alpha=0.1, \beta=4.0$  & $3.48 \times 10^{-4}$ & $3.37 \times 10^{-4}$ & $3.79 \times 10^{-5}$ & $6.97 \times 10^{-5}$ & $7.99 \times 10^{-5}$ \\
$\alpha=0.1, \beta=4.5$  & $3.85 \times 10^{-4}$ & $3.62 \times 10^{-4}$ & $5.74 \times 10^{-5}$ & $\mathbf{4.23 \times 10^{-5}}$ & $5.17 \times 10^{-5}$ \\
$\alpha=0.1, \beta=5.0$  & $7.09 \times 10^{-4}$ & $5.39 \times 10^{-4}$ & $1.32 \times 10^{-4}$ & $1.21 \times 10^{-4}$ & $9.72 \times 10^{-5}$ \\
\bottomrule
\end{tabular}
}
\end{table}

\begin{table}[htbp]
\centering
\caption{Ablation study results of relative errors at $\mathrm{Kn} = 1.0$ for the standard $L^2$ Loss, relative loss~\cite{oh2025separable}, and the proposed loss with specific choices of $\alpha$ and $\beta$.}
\label{tab:ablation_kn1}
\resizebox{\textwidth}{!}{%
\setlength{\tabcolsep}{16.1pt}%
\begin{tabular}{lccccc}
\toprule
\textbf{Loss function} & $\boldsymbol{f}$ \textbf{(Relative $L^2$)} & $\boldsymbol{f}$ \textbf{(Relative $L^1$)} & $\boldsymbol{\rho}$ \textbf{(Relative $L^1$)} & $\boldsymbol{u_x}$ \textbf{(Relative $L^1$)} & $\boldsymbol{T}$ \textbf{(Relative $L^1$)} \\
\midrule
Standard $L^2$ Loss       & $1.49 \times 10^{-3}$ & $2.89 \times 10^{-3}$ & $1.03 \times 10^{-3}$ & $1.82 \times 10^{-3}$ & $4.59 \times 10^{-3}$ \\
Relative Loss~\cite{oh2025separable}  & $2.64 \times 10^{-3}$ & $2.37 \times 10^{-3}$ & $4.54 \times 10^{-4}$ & $1.33 \times 10^{-3}$ & $2.45 \times 10^{-3}$ \\
\midrule
$\alpha=0.01, \beta=5.0$ & $\mathbf{4.48 \times 10^{-4}}$ & $6.82 \times 10^{-4}$ & $5.73 \times 10^{-5}$ & $1.44 \times 10^{-4}$ & $1.99 \times 10^{-4}$ \\
$\alpha=0.01, \beta=5.5$ & $5.18 \times 10^{-4}$ & $\mathbf{6.51 \times 10^{-4}}$ & $4.35 \times 10^{-5}$ & $\mathbf{7.62 \times 10^{-5}}$ & $\mathbf{1.03 \times 10^{-4}}$ \\
$\alpha=0.01, \beta=6.0$ & $6.43 \times 10^{-4}$ & $7.17 \times 10^{-4}$ & $5.30 \times 10^{-5}$ & $8.41 \times 10^{-5}$ & $1.04 \times 10^{-4}$ \\
$\alpha=0.01, \beta=6.5$ & $1.22 \times 10^{-3}$ & $1.11 \times 10^{-3}$ & $1.38 \times 10^{-4}$ & $8.53 \times 10^{-5}$ & $1.60 \times 10^{-4}$ \\
$\alpha=0.01, \beta=7.0$ & $2.78 \times 10^{-3}$ & $2.10 \times 10^{-3}$ & $3.00 \times 10^{-4}$ & $2.00 \times 10^{-4}$ & $1.51 \times 10^{-4}$ \\
\midrule
$\alpha=0.05, \beta=3.5$ & $4.53 \times 10^{-4}$ & $7.96 \times 10^{-4}$ & $5.29 \times 10^{-5}$ & $2.87 \times 10^{-4}$ & $3.15 \times 10^{-4}$ \\
$\alpha=0.05, \beta=4.0$ & $4.62 \times 10^{-4}$ & $6.67 \times 10^{-4}$ & $4.68 \times 10^{-5}$ & $1.16 \times 10^{-4}$ & $2.31 \times 10^{-4}$ \\
$\alpha=0.05, \beta=4.5$ & $6.40 \times 10^{-4}$ & $7.15 \times 10^{-4}$ & $\mathbf{4.35 \times 10^{-5}}$ & $8.54 \times 10^{-5}$ & $1.56 \times 10^{-4}$ \\
$\alpha=0.05, \beta=5.0$ & $7.31 \times 10^{-4}$ & $7.63 \times 10^{-4}$ & $8.54 \times 10^{-5}$ & $7.88 \times 10^{-5}$ & $1.35 \times 10^{-4}$ \\
$\alpha=0.05, \beta=5.5$ & $1.37 \times 10^{-3}$ & $1.12 \times 10^{-3}$ & $1.78 \times 10^{-4}$ & $1.12 \times 10^{-4}$ & $1.35 \times 10^{-4}$ \\
\midrule
$\alpha=0.1, \beta=3.5$  & $4.55 \times 10^{-4}$ & $6.82 \times 10^{-4}$ & $8.31 \times 10^{-5}$ & $1.46 \times 10^{-4}$ & $3.48 \times 10^{-4}$ \\
$\alpha=0.1, \beta=4.0$  & $6.68 \times 10^{-4}$ & $7.46 \times 10^{-4}$ & $4.48 \times 10^{-5}$ & $9.03 \times 10^{-5}$ & $2.11 \times 10^{-4}$ \\
$\alpha=0.1, \beta=4.5$  & $8.35 \times 10^{-4}$ & $7.91 \times 10^{-4}$ & $8.00 \times 10^{-5}$ & $7.84 \times 10^{-5}$ & $1.68 \times 10^{-4}$ \\
$\alpha=0.1, \beta=5.0$  & $1.43 \times 10^{-3}$ & $1.11 \times 10^{-3}$ & $1.22 \times 10^{-4}$ & $1.72 \times 10^{-4}$ & $1.91 \times 10^{-4}$ \\
\bottomrule
\end{tabular}
}
\end{table}

It is notable that for each $\alpha$, the error curves clearly exhibit a U-shaped pattern, where the error decreases to a certain point and then increases. When the polynomial growth rate $\beta$ is small, the weight function may fail to sufficiently control the PDE residual in the high-velocity tail. Consequently, the accuracy can degrade. Conversely, when $\beta$ becomes excessively large, the errors increase sharply, \rev{in some metrics even surpassing those of the unweighted standard PINN (e.g., $\rho$ at $\beta = 6.0$ in Table~\ref{tab:ablation_kn001})}. This performance degradation does not contradict our theoretical results. Theorem~\ref{thm:main_stability} guarantees small errors only provided that the loss function is sufficiently minimized. Such an excessive weight assigned to the tail regions can introduce extreme bias in the gradients during neural network optimization, thereby resulting in increased loss and training instability. We also note that decreasing $\alpha$ necessitates larger values of $\beta$ to achieve the same level of accuracy. \rev{The same U-shaped behavior persists across all three Knudsen numbers.}

At $\mathrm{Kn} = 1.0$, the standard $L^2$ loss inherently exhibits reasonably good performance. Nevertheless, as shown in Figure~\ref{fig:ablation_curves_kn1}, the proposed method generally yields even better results across most combinations of $\alpha$ and $\beta$.

Although the specific choice of $\alpha$ and $\beta$ affects the overall performance, it is evident that most choices of $\alpha$ and $\beta$ generally yield better results than the standard $L^2$ and relative loss baselines. \rev{In particular, the setting $(\alpha, \beta) = (0.1, 4.0)$ employed throughout the paper performs reliably in all of the tested flow configurations, which confirms that the reported improvements do not depend on a delicate tuning of the weight.}

\end{document}